\documentclass[lettersize,journal]{IEEEtran}
\usepackage[
colorlinks=true,    
linkcolor=blue,     
citecolor=green,    
filecolor=magenta,  
urlcolor=cyan,      
pdfborder={0 0 0},  
bookmarks=true,     
pdfauthor={Your Name},
pdftitle={Your Title},
]{hyperref}

\IEEEoverridecommandlockouts
\usepackage{cite}
\usepackage{amsmath,amssymb,amsfonts}
\usepackage{algorithmic}
\usepackage{graphicx}
\usepackage{textcomp}
\usepackage{xcolor}
\usepackage{multirow}
\usepackage[ruled,vlined]{algorithm2e}
\usepackage{hyperref}
\usepackage{subcaption}
\usepackage{booktabs} 
\usepackage{multicol}
\usepackage[numbers,sort&compress]{natbib}
\usepackage{amsthm}
\usepackage{cleveref}
\usepackage{float}
\usepackage{stfloats}
\usepackage{url}
\usepackage{enumitem}
\usepackage{threeparttable}
\usepackage[numbers,sort&compress]{natbib}

\newtheorem{theorem}{Theorem}

\newtheorem{definition}{Definition}
\newtheorem{assumption}{Assumption}
\newtheorem{lemma}{Lemma}
\newtheorem{example}{Example}
\newtheorem{basic-eq}{}
\newtheorem{proofpart}{Part}

\def\w{\mathbf{w}}
\def\v{\mathbf{v}}
\def\W{\mathbf{W}}
\def\A{\mathbf{A}}

\def\P{\mathbf{P}}
\def\Q{\mathbf{Q}}

\def\U{\mathbf{U}}
\def\I{\mathbf{I}} 
\def\Ps{\mathbf{\Psi}}
\def\PO{\mathbb{P}_{\Omega}}
\def\om{\mathbf{1}_m}
\def\FE{{CEPS}}

\graphicspath{ {./Figures/} }

\makeatletter
\def\endthebibliography{%
	\def\@noitemerr{\@latex@warning{Empty `thebibliography' environment}}%
	\endlist
}
\makeatother

\begin{document}
\flushbottom

	\title{Sparse Decentralized Federated Learning}
	
	\author{Shan Sha,  Shenglong Zhou, Lingchen Kong, Geoffrey Ye Li, \IEEEmembership{Fellow, IEEE}
		\thanks{This work was supported by the Fundamental Research Funds for the Central Universities and the Talent Fund of Beijing Jiaotong University. {\it Corresponding author: Shenglong Zhou}.}%
		\thanks{Shan Sha, Lingchen Kong, and Shenglong Zhou are with the School of Mathematics and Statistics, Beijing Jiaotong University, Beijing, China. E-mail: \{20118023, lchkong, shlzhou\}@bjtu.edu.cn. Geoffrey Ye Li are with the Department of Electrical and Electronic Engineering, Faculty of Engineering, Imperial College London, London, U.K. E-mail: geoffrey.li@imperial.ac.uk.}%
	}

	
	\IEEEpubid{}
	
	\maketitle
	
	\begin{abstract}
		Decentralized Federated Learning (DFL) enables collaborative model training without a central server but faces challenges in efficiency, stability, and trustworthiness due to communication and computational limitations among distributed nodes. To address these critical issues, we introduce a sparsity constraint on the shared model, leading to Sparse DFL (SDFL), and propose a novel algorithm, CEPS. The sparsity constraint facilitates the use of one-bit compressive sensing to transmit one-bit information between partially selected neighbour nodes at specific steps, thereby significantly improving communication efficiency. Moreover, we integrate differential privacy into the algorithm to ensure privacy preservation and bolster the trustworthiness of the learning process. Furthermore, CEPS is underpinned by theoretical guarantees regarding both convergence and privacy. Numerical experiments validate the effectiveness of the proposed algorithm in improving communication and computation efficiency while maintaining a high level of trustworthiness.  

	\end{abstract}


	\section{Introduction}
	\IEEEPARstart{F}{ederated} learning (FL \cite{mcmahan2017communication}, \cite{kairouz2021advances}) is a machine learning framework that enables collaborative model training across multiple nodes (e.g., devices or entities) by exchanging local parameters instead of raw data. 
	This approach has gained widespread attention with the fast development of Internet of Things, where data are stored in distributed devices and it is impossible to collect raw data as in traditional machine learning settings due to the restriction of communication resource, data privacy concern, or law regulations.  
	In the most predominate FL setting, a central server or coordinator is required to aggregate parameters from all nodes, update a global model, and broadcast it back to the nodes for further training, which is referred to as centralized federated learning (CFL). 
	In contrast, decentralized federated learning (DFL \cite{beltran2023decentralized}, \cite{yuan2024decentralized}), as a new setting of FL, eliminates the need for a central server, allowing nodes to communicate directly with their neighbours based on a predefined network topology and perform local update. DFL addresses the communication bottleneck inherent in CFL by distributing information exchange among all nodes, which avoids the limitation of single-point failure and improves fault tolerance. Moreover, DFL can enhance the robustness of the learning process because nodes constantly update their knowledge within the communication network and can achieve comparable or even superior accuracy and computational efficiency compared to CFL \cite{lian2017can}.
	
	Despite these advantages, the decentralized setting of DFL also introduces new challenges. Communication efficiency remains a primary concern, involving the selection of neighbour nodes for communication and the transmission of information without a central server. The absence of the central aggregator complicates theoretical analysis, as it requires managing multiple local models at each step instead of a single global model. Trustworthiness is another critical issue, although FL inherently prevents raw data sharing, it has been shown in \cite{mothukuri2021survey} that local model parameters can still leak private information. Therefore, implementing privacy-preserving techniques is essential in FL, and even more so in DFL, where decentralized communication is more susceptible to attacks compared to communication with a protected central server.
	
	\begin{table*}[t]
		\centering\renewcommand\arraystretch{1.1}\addtolength{\tabcolsep}{2pt}
				\caption{Comparisons of Different Algorithms.}
		\label{algcompare}
		\begin{tabular}{lllllll}
			\hline
			& Ref. & Local Updates & Device Participation & Communication Topology & Information Compression & Privacy Protection \\ \hline
			D-PSGD & \cite{lian2017can} & Single & Full & Mixing Matrix & None & None \\ 
			DCD-PSGD & \cite{tang2018communication} & Single & Full & Mixing Matrix & Difference Compression & None \\
			D-FedAvgM & \cite{sun2022decentralized} & Multiple & Full & Mixing Matrix & Quantization & None \\
			DFedSAM & \cite{shi2023improving} & Single & Full & Mixing Matrix & None & None \\
			MATCHA & \cite{wang2019matcha} & Multiple & Partial & Dynamic Network & None & None \\
			DP-DSGT & \cite{bayrooti2023differentially} & Single & Full & Mixing Matrix & None & Differential Privacy \\
			{CEPS} & Our & Multiple & Partial & Dynamic Network & 1BCS & Differential Privacy \\
			\hline
		\end{tabular}
	\end{table*}	
	\subsection{Prior Arts}
	We begin by reviewing foundational DFL algorithms and then discuss important DFL topics such as communication efficiency, communication topology, and privacy preservation.
	\subsubsection{Foundational DFL Algorithms}
	 A decentralized parallel stochastic gradient descent (D-PSGD) method has been introduced by \cite{lian2017can} to address DFL by performing SGD scheme synchronously and parallelly across all nodes, with the communication protocol defined by a symmetric doubly stochastic matrix. Subsequently, various methods based on this decentralized SGD scheme have been proposed.
	\cite{lalitha2018fully} presented a Bayesian-like approach by introducing of a belief over the model parameter space on local nodes, enabling nodes to communicate with fewer neighbours.
	\cite{pappas2021ipls} offered a decentralized framework based on the Interplanetary File System to overcome limitations of CFL, such as single points of failure and bandwidth bottlenecks.
	\cite{ye2022decentralized} developed a robust solution for DFL in unreliable communication environments by updating model parameters with partially received messages and optimizing mixing weights.
	To address the overfitting issue inherent in DFL, \cite{shi2023improving} incorporated Sharpness-Aware Minimization to improve generalization performance and considers multiple inter-node communications to enhance model consistency. Other related research includes \cite{li2021decentralized, kalra2023decentralized}.
	
	\subsubsection{Communication Efficiency}
	Communication efficiency is a primary goal in federated learning research, whether in centralized or decentralized settings. Two main strategies to reduce communication costs are data compression and reducing the number of communication rounds.
	In \cite{sun2022decentralized}, a decentralized version of the federated averaging algorithm with momentum (DFedAvgM) has been proposed. It performed multiple local updates with momentum and included one communication step per round to enhance efficiency. A quantized version of DFedAvgM was also investigated to reduce communication overhead further.
	\cite{liu2022decentralized} proposed a general framework of DFL to achieve a better balance between communication efficiency and model consensus with a similar strategy.
	Several works have explored different compression strategies in DFL. In \cite{tang2018communication}, two variants based on D-PSGD were proposed to enhance communication efficiency using extrapolation compression and difference compression strategies. \cite{koloskova2019decentralized} and \cite{koloskovadecentralized} provided a general quantized communication framework for DFL with arbitrary compression operators. \cite{dai2022dispfl} also discussed communication-efficient methods in DFL.
	
	\subsubsection{Communication Topology}
	Communication topology significantly affects the convergence speed and efficiency of DFL, as it determines which neighbours each node communicates within each round. Optimizing the communication topology can lead to faster convergence and reduced communication overhead.
	MATCHA \cite{wang2019matcha} speeded up decentralized SGD by prioritizing critical communication links and allowing parallel communication over disjoint links to minimize communication delays while maintaining fast convergence rates. GossipFL \cite{tang2022gossipfl} proposed a sparsification model to reduce communication traffic and introduced an adaptive peer-selection mechanism to improve bandwidth utilization while preserving model convergence. Other works such as \cite{hegedHus2019gossip, hegedHus2021decentralized} also addressed the impact of communication topology on DFL performance.
	\subsubsection{Privacy Preservation}
	Privacy preservation is a key feature of FL, ensuring that raw data is not shared with a central server or other nodes. However, if transmitted information (e.g., gradients) is intercepted by adversaries, the raw data can be reconstructed, leading to privacy breaches \cite{zhu2019deep, huang2021evaluating}. The importance of privacy preservation, along with other trustworthy artificial intelligence (AI) components such as fairness and robustness, is increasingly recognized in the development of secure and reliable AI systems (\cite{hu2020personalized,ur2021trustfed,ali2022federated}).
	Existing work in FL mainly focuses on deploying privacy-preserving techniques like differential privacy (DP) in centralized FL \cite{wei2020federated, naseri2020local, ryu2022differentially, lowy2023private}. In the context of DFL, research on privacy preservation is relatively limited. \cite{wang2023decentralized} proposed a differential privacy decentralized optimization algorithm based on SGD, providing privacy guarantees without relying on a central server. \cite{wang2023tailoring} considered static and dynamic consensus-based gradient methods where DP is combined with gossip algorithms to ensure privacy. \cite{bayrooti2023differentially} provided a generalization of D-PSGD with privacy guarantee.

	\subsection{Our contribution}

	The main contribution of this paper is the development of a novel SDFL algorithm, CEPS (standing for a \textbf{C}ommunication- \textbf{E}fficient and \textbf{P}rivate method for \textbf{S}DFL), designed to address multiple critical challenges in DFL. Table \ref{algcompare} compares CEPS with six DFL algorithms, highlighting its ability to deal with a wider range of challenges simultaneously. This capability stems from several key advantages outlined below.
	\begin{itemize}[leftmargin=12pt]
		\item \textit{Communication efficiency.} Due to the integration of a sparsity constraint in the learning model, we harness the 1BCS technique to exchange one-bit signals when two neighbour nodes communicate with each other. Consequently, the transmitted content is significantly lessened. To the best of our knowledge, we have not come across similar efforts to incorporate 1BCS into DFL. Simultaneously, each node communicates with its neighbours only at specific steps rather than every step to update its parameters, thereby greatly reducing the number of communication rounds.
		\item \textit{Computational efficiency.}  It results from two factors. Every node solves subproblems without much computational endeavour using closed-form solutions. Complex elements, such as gradients, are computed at designated steps and remain unchanged during other steps.
		\item \textit{Robustness against stragglers.} Each node can selectively pick partial neighbour nodes to collect messages to update its parameter, so neighbours who suffer from unreliable and imperfect communication links can be paid less attention.
		\item \textit{Privacy Protection.} By incorporating the Gaussian mechanism, we provide a rigorous privacy guarantee for data across all nodes under a certain noise perturbation level.
		\item \textit{Theoretical Guarantee.}  We provide a convergence bound of {CEPS} under common assumptions in DFL analysis, which ensures the algorithm gives an approximation solution in terms of the average global gradient.
		\item \textit{Desirable numerical performance.} Numerical experiments have demonstrated that the 1BCS technique enables effective communication and thus allows the proposed algorithm to achieve desirable training accuracy. We also compare {CEPS} with several leading DFL algorithms, highlighting its great potential in dealing with the challenges of imperfect and unreliable communication environments.	
	\end{itemize}
	
		We point out that part of our work has been previously published in \cite{zhou2024communication}. In this paper, we extend the original study with three major contributions: First, we enhance the algorithm by incorporating DP to ensure privacy protection. Second, we provide rigorous convergence proofs, offering strong theoretical guarantees. Third, extensive numerical experiments and comparisons with other DFL algorithms validate the effectiveness and performance of our approach. These improvements significantly advance the initial ideas and demonstrate the enhanced capabilities of  CEPS.	
	\subsection{Organization}
	The paper is organized as follows. Section \ref{Preliminaries} presents all notation, the models of interest, and the definition of DP. In Section \ref{DFL}, we develop the algorithm and highlight its several advantageous properties. In section \ref{Theory}, we provide theoretical analysis including the convergence and privacy guarantee.
	In Section \ref{Numerical}, we conduct some numerical experiments to demonstrate the performance of the proposed algorithm. Concluding remarks are given in the last section.

	\section{Preliminaries}
	\label{Preliminaries}
	In this section, we introduce the notation used in this paper,  the problems of DFL and SDFL, and the definition of DP.
	\subsection{Notation}
	We use plain, bold, and capital letters to present scalars, vectors, and matrices, respectively, e.g., $n$ and $\sigma$ are scalars, $\mathbf{w}$ and $\boldsymbol{\xi}$ are vectors, $\mathbf{W}$ and $\mathbf{\Xi}$ are matrices. Let $\mathbb{R}^n$ denote the $n$-dimensional Euclidean space equipped with inner product $\langle\cdot, \cdot\rangle$, defined by ${\langle\mathbf{w}, \mathbf{v}\rangle:=\sum_{i=1}^{m} w_i v_i}$, where `$:=$' means `define'. We use $\|\cdot\|$ to denote the Euclidean norm for vectors and the Frobenius norm for matrices, namely, ${\|\mathbf{w}\| = \sqrt{\sum_{i} w_i^2}}$  and ${\|\W\| = \sqrt{\sum_{ij}  W_{ij}^2}}$. 
	The sign function is defined as ${\operatorname{sign}(t)=1}$ if ${t>0}$ and $-1$ otherwise and is applied element-wisely for vectors: $\operatorname{sign}(\mathbf{w}) = \left(\operatorname{sign}(w_1), \operatorname{sign}(w_2), \dots, \operatorname{sign}(w_n)\right)^\top$, where $^\top$ represents the transpose.  The projection of a point $\mathbf{w}$ onto $\Omega$ is defined by
	$$
	\mathbb{P}_{\Omega}(\mathbf{w}) \in \operatorname{argmin}_{\mathbf{z} \in \Omega}\|\mathbf{z}-\mathbf{w}\|^2,
	$$
	which keeps the first $s$ largest (in absolute value) entries in $\Omega$ and sets the remaining entries to be zero. Finally, by letting $\mathbf{W}:=\left(\mathbf{w}_1, \mathbf{w}_2, \cdots, \mathbf{w}_m\right)$, the sparse projection operator on matrix $\mathbf{W}$ is defined as $$\mathbb{P}_{\Omega}(\mathbf{W}) = \left(\mathbb{P}_{\Omega}(\mathbf{w}_1),\mathbb{P}_{\Omega}(\mathbf{w}_2),\cdots,\mathbb{P}_{\Omega}(\mathbf{w}_m)\right).$$
	
	\subsection{Decentralized Federated Learning}
	Given $m$ nodes
	${{V}:=\{1,2, \cdots, m\}}$, each node $i$ has a loss function ${f_i(\cdot):=f_i\left(\cdot ; {D}_i\right)}$ associated with private dataset ${D}_i$, where ${f_i: \mathbb{R}^n \rightarrow \mathbb{R}}$ is continuous and bounded from below. In a DFL scenario,  a general setting of communication topology can be given as follows: Each node $i$ only exchanges messages with its neighbours ${j \in {N}_i}$, where ${{N}_i\subseteq {V}}$ consists of node $i$ and the nodes connected to node $i$. Overall, we obtain a graph ${{G} = ({V}, {E})}$, where ${{E} = \{(i, j): j \in {N}_i, i \in {V}\}}$ is the set of edges. Based on this setting, the target graph is undirected, namely if ${i\in{N}_j}$ then ${j\in{N}_i}$. Hereafter, we always assume graph ${G}$ is connected. The task of DFL is to train a model by solving the following optimization problem	 
	\begin{equation}\label{objfunc}
	\begin{aligned}
	f^*:=\min _{\mathbf{w}_i \in \mathbb{R}^n}~& \sum_{i \in {V}} f_i(\mathbf{w}_i),\\
		\text { s.t. } ~& \mathbf{w}_i=\mathbf{w}_j, ~(i, j) \in {E},
	\end{aligned}\tag{DFL}
	\end{equation}
	where ${\mathbf{w}_i, i \in {V}}$ are the parameters to be learned. Since graph ${G}$ is connected, it follows $\mathbf{w}_1=\mathbf{w}_2=\cdots=\mathbf{w}_m$. As each $f_i$ is bounded from below, we have ${f^*>-\infty}.$ Under the setting of DFL, each node trains its own local model $\mathbf{w}_i$, and the model consensus is realized by communication among nodes.

	\subsection{Sparse DFL}
	 The objective of our work is to train a model by solving the following sparsity constrained DFL model,
	\begin{equation}\label{dfl}
		\begin{aligned}
			\min _{\mathbf{W}} ~& \sum_{i \in {V}} f_i\left(\mathbf{w}_i\right) \\
			\text { s.t. } ~&  \mathbf{w}_i=\mathbf{w}_j, ~(i, j) \in {E},~~\mathbf{w}_i \in \Omega,~i \in {V}, 
		\end{aligned}\tag{SDFL}
	\end{equation}
	where ${\mathbf{W}:=\left(\mathbf{w}_1, \mathbf{w}_2, \cdots, \mathbf{w}_m\right) \in \mathbb{R}^{n \times m}}$ and $\Omega$ is a sparsity constraint defined by
	$$
	\Omega:=\left\{\mathbf{w} \in \mathbb{R}^n:\|\mathbf{w}\|_0 \leq s\right\}.
	$$
	Here, $\|\mathbf{w}\|_0$ is known as a zero norm that counts the number of non-zero entries of $\mathbf{w}$ and $s \ll n$ is a given sparsity level.   The motivation for studying \eqref{dfl} arises from two aspects. First, the incorporation of a sparsity constraint drastically reduces the volume of transmitted messages, thereby significantly enhancing communication efficiency, as demonstrated in Section \ref{Com-eff} and supported by our numerical experiments (see Fig. \ref{example1_data}). Second, model \eqref{dfl} is a special case of sparse optimization, rooted in the principles of compressive sensing \cite{candes2005decoding, donoho2006compressed}. As stated in  \cite{bach2012optimization, baraniuk2010applications, zhou2021global},  leveraging sparsity provides computationally feasible ways to reveal the structure of massive high-dimensional data, giving rise to a large repertoire of efficient algorithms for fast computation. Overall, \eqref{dfl} ensures both high communication efficiency and computational feasibility.

	\subsection{Differential Privacy}
	DP is intrinsically a property of a randomized mechanism \cite{dwork2006calibrating}.  We adopt the definition of DP from \cite{abadi2016deep}. 
	\begin{definition} (Differential Privacy)
		A randomized mechanism: ${\mathcal{M}: {D} \rightarrow \mathbb{R}}$ is $(\epsilon, \delta)$-differentially private if, for any two adjacent datasets $H,H' \subseteq {D}$ and any $S \subseteq \mathbb{R}$ satisfies
		\begin{equation}
			\operatorname{Pr}[\mathcal{M}(H) \in S] \leq e^{\varepsilon} \operatorname{Pr}\left[\mathcal{M}\left(H'\right) \in S\right]+\delta,
			\label{DP}
		\end{equation}
		where adjacent datasets refer to $H$ and $H'$ are at the same size and differ in exactly one record. 
	\end{definition}	
	From the perspective of algorithmic design, an algorithm is of DP if it generates a randomized response that satisfies condition (\ref{DP}). This ensures that the outputs are statistically similar for adjacent datasets, thereby preventing adversaries from inferring specific information about individual data points.		
	The DP of an algorithm is typically achieved by adding noise drawn from a specific probability distribution. The magnitude of this noise depends on the functional sensitivity, which quantifies how much the output can change in response to variations in the input. 
	\begin{definition} (Function Sensitivity)
		The global sensitivity of an algorithm (function) $
		{\mathcal{A}:{D} \rightarrow \mathbb{R}}$ is the maximum output difference for any two adjacent datasets $H$ and $H'$,
		\begin{equation}
			\eta(\mathcal{A})=\max _{\substack{H,H' \subseteq {D} }}\|\mathcal{A}(H)-\mathcal{A}(H')\|_2 .
			\label{GS}
		\end{equation}
	\end{definition}
	It is obvious from (\ref{DP}) and (\ref{GS}) that smaller values of $\varepsilon$ and $\delta$ provide stronger privacy protection, while a larger $\eta(\mathcal{A})$ indicates higher sensitivity to input changes, necessitating more noise. A popular mechanism $\mathcal{M}$ to ensure DP of algorithm $\mathcal{A}$  is the Gaussian mechanism, 
	\begin{equation*}
		\mathcal{M}({D})=\mathcal{A}({D})+\mathcal{N}\left(0, \frac{2 \ln (1.25 / \delta)}{\varepsilon^2}\eta^2(\mathcal{A})\right),
	\end{equation*}
	where $\mathcal{N}(0, \sigma^2)$ denotes a Gaussian distribution with mean zero and variance $\sigma^2$.

	\section{DFL via Inexact ADM}
	\label{DFL}

	\begin{algorithm*} [!th]
		Initialize $\mathbf{W}^0=0$ and $\mu>0$.
		
		\For{every node $i\in{V}=\{1,2,\cdots,m \}$}{  
			Initialize  an integer $ \kappa_i>0$, three real numbers $\varrho_i>0, \gamma_i>0, \sigma_i >0$, and an encoding matrix $\mathbf{\Phi}_i\in\mathbb{R}^{d_i\times n}$. 
			
			Initialize  $m_i^{-1}=|N_i|$ and $\mathbf{u}_i^{-1}=-\nabla f_i(\w_i^0)$. 
			
			Share  $\gamma_i$ and $\mathbf{\Phi}_i$ with neighbours nodes in ${N}_i$ through certain agreements.
			
		} 
		
		\For{$k=0,1,2,3,\cdots $}{
			
			\For{every node $i\in{V}$}{ 				
			
				\If{$k \in\mathcal{K}_i :=\{\kappa_i,2\kappa_i,3\kappa_i,\cdots \}$} 
				{
					
					 Randomly select a subset ${N}_i^{k}\subseteq {N}_i$ and  send request to neighbour nodes in ${N}_i^{k}$. 
					

					Employ Algorithm \ref{algorithm-1bcs} to recover $\mathbf{z}_j^{k}$ by $\mathbf{z}_j^{k}= 
					\texttt{1BCS}(\mathbf{w}_j^{k},\mathbf{\Phi}_j,\gamma_j)$ for each neighbour $j\in{N}_i^{k}$.

					Generate noise $\boldsymbol{\xi}_i^{k}  \sim \mathcal{N}\left(\mathbf{0}, \varrho_i \mathbf{I}\right)$ and update parameters by
					\begin{eqnarray}\label{def-w-bar-tilde} m_i^{k}=\left|{N}_i^k\right|,~~~\overline{\mathbf{w}}_i^{k}=\frac{1}{m_i^{k}} \sum_{j\in{N}_i^{k}}\mathbf{z}^{k}_j,~~~\mathbf{u}_i^{k}= \sigma_i m_i^{k} \overline{\mathbf{w}}_i^k-\nabla f_i(\overline{\mathbf{w}}_i^k),~~~{\mathbf{w}}_{i}^{k+1}=\mathbb{P}_{\Omega} \left(\frac{ {\mathbf{u}}_{i}^{k}+\boldsymbol{\xi}_i^{k} }{ \sigma_i m_i^{k} }\right).
					\end{eqnarray} 
				}
				\Else{
				\begin{equation}\label{iceadmm-sub5}
					\begin{aligned}
						m_i^{k}=m_i^{k-1},\qquad {\mathbf{u}}_{i}^{k}=  {\mathbf{u}}_{i}^{k-1},\qquad {\mathbf{w}}_{i}^{k+1}=
						 \mathbb{P}_{\Omega} \left(\frac{ {\mathbf{u}}_{i}^{k}+\mu \mathbf{w}^{k}_i}{ \sigma_i m_i^{k} +\mu}\right).
					\end{aligned}\end{equation}				
				}
												 
			}
		}
		\caption{{\FE }: \textbf{C}ommunication-\textbf{E}fficient and \textbf{P}rivate method for \textbf{S}DFL.
			\label{algorithm-{CEPS}}}
		
	\end{algorithm*} 
	
	\subsection{Algorithm Design}
	 To address  (\ref{dfl}), we solve its personalized version,
	\begin{equation}
		\begin{aligned}
			\min _{\mathbf{W}} ~~& \sum_{i \in {V}}\Big(f_i\left(\mathbf{w}_i\right)+\frac{\sigma_i}{2} \sum_{j \in {N}_i}\left\|\mathbf{w}_i-\mathbf{w}_j\right\|^2\Big), \\
			\text { s.t. } ~~& \mathbf{w}_i \in \Omega,~~~ i \in {V},
		\end{aligned}
	\end{equation}
	where ${\sigma_i>0}$. We adopt the inexact alternating direction method (iADM) to solve the above problem. To be more specific, when point $\mathbf{W}^k=\left(\mathbf{w}_1^k, \mathbf{w}_2^k, \cdots, \mathbf{w}_m^k\right)$ is obtained, we update $\mathbf{w}_i^{k+1}$ for $i \in {V}$ by solving the following problem,
	\begin{equation}\label{subproblem-wk1}
		\min _{\mathbf{w}_i \in \Omega}~~ f_i\left(\mathbf{w}_i\right)+ \frac{\sigma_i}{2} \sum_{j \in {N}_i^k}\left\|\mathbf{w}_i-\mathbf{w}_j^k\right\|^2,
	\end{equation}
	where {${N}_i^k\subseteq{N}_i$} is a randomly selected subset and also contains node $i$. 	The overall algorithmic framework is presented in  Algorithm \ref{algorithm-{CEPS}}.  We would like to explain why we update $\w_i^k$ by \eqref{def-w-bar-tilde} and \eqref{iceadmm-sub5}. Let $a\in\{1,2,3,\ldots\}$. For each node $i$, we consider $\kappa_i$ consecutive iterations: $$(a-1)\kappa_i+1, ~(a-1)\kappa_i+2, ~\ldots,  a\kappa_i-1, ~a\kappa_i.$$ 
	\begin{itemize}[leftmargin=12pt]
		\item  When  $k=(a-1)\kappa_i+1,(a-1)\kappa_i+2,\ldots,a\kappa_i-1$ where no  communication occurs between node $i$ and its neighbour nodes $j\in{N}_i^k\equiv {N}_i^{(a-1)\kappa_i}$, we update $\mathbf{w}_i^{k+1}$ by  solving 
			\begin{align}		\label{update-w-no-commu}
				\mathbf{w}_i^{k+1}
				=~&\underset{\mathbf{w}_i\in\Omega}{\rm argmin}~    \left\langle\nabla f_i\left(\overline{\mathbf{w}}_i^{(a-1)\kappa_i}\right), \mathbf{w}_i \right\rangle\nonumber \\[1.25ex]
				+~&\frac{\sigma_i}{2} \sum_{j \in {N}_i^{k}} \left\|\mathbf{w}_i-\mathbf{z}_j^{(a-1)\kappa_i}\right\|^2+\frac{\mu}{2}\left\|\mathbf{w}_i- {\mathbf{w}}_i^{k}\right\|^2  \nonumber \\
				=~&\mathbb{P}_{\Omega} \left(\frac{\mathbf{u}_i^{(a-1)\kappa_i} +\mu \mathbf{w}^{k}_i}{ \sigma_i m_i^{(a-1)\kappa_i} +\mu}\right),
			\end{align}
	where $\overline{\mathbf{w}}_i^{k}$ and $\mathbf{u}_i^k$ are defined in \eqref{def-w-bar-tilde}, and $\mathbf{z}_j^{k}$ is the recovered version of $\mathbf{w}_j^k$ by node $i$ from node $j$. This leads to    \eqref{iceadmm-sub5}. One can observe that  for these ${\kappa_i-1}$ iterations, $\mathbf{u}_i^{(a-1)\kappa_i}$	is fixed, lessening the computational complexity. 
					\item  When  $k=a\kappa_i$ where communication occurs between node $i$ and node  $j\in{N}_i^k$, we update $\mathbf{w}_i^{k+1}$ by  solving 

			\begin{align}		\label{update-w-commu}
				\mathbf{w}_i^{k+1}
				=&\underset{\mathbf{w}_i\in\Omega}{\rm argmin}    \left\langle\nabla f_i\left(\overline{\mathbf{w}}_i^{k}\right)-\boldsymbol{\xi}_i^{k} , \mathbf{w}_i \right\rangle + \frac{\sigma_i}{2} \sum_{j \in {N}_i^{k}} \left\|\mathbf{w}_i-\mathbf{z}_j^{k}\right\|^2 \nonumber \\
				=&\mathbb{P}_{\Omega} \left(\frac{\mathbf{u}_i^{k} +\boldsymbol{\xi}_i^{k} }{ \sigma_i m_i^{k} }\right),
			\end{align}
	where $\boldsymbol{\xi}_i$ is the noise which aims to  perturb $\nabla f_i(\overline{\mathbf{v}}_i^k)$ so as to ensure the data privacy for node $i$.  This gives rise to updating $\mathbf{w}_i^{k+1}$ by  \eqref{def-w-bar-tilde}.			
	\end{itemize}

	\subsection{One-bit Compressive Sensing} 
	To ensure the great efficiency of communication among neighbour nodes, our aim is to exchange information in the form of 1-bit data. However, recovering accurate information (i.e., the trained local parameter) from 1-bit observations is quite challenging without additional information. To break through this limitation, we leverage the technique of 1BCS, which involves two phases.
	\begin{itemize}[leftmargin=13pt] 
		\item \textit{Phase I: Encoding.} Suppose each node ${i \in {V}}$ has a encoding matrix ${\boldsymbol{\Phi}_i \in \mathbb{R}^{d_i \times n}}$, where ${d_i>0}$. This coding matrix is only accessible to node $i$ 's neighbours ${N}_i$ through certain agreements. When a parameter ${\mathbf{w}_j \in \Omega}$ is trained at node $j \in {N}_i$, node $j$ encodes it as presented in Algorithm \ref{algorithm-1bcs}, where ${\gamma_j>0}$ (e.g.,  $10$ or $100$) and $\odot$ is the Hadamard product. We point out Step 2 in Algorithm 2 aims to rescale $\mathbf{w}_j$ so that $\mathbf{x}_j$ does not have tiny magnitudes of non-zero entries. It is recognized that small values in the magnitude of $\mathbf{w}_j$ can cause failures of many algorithms to recover $\mathbf{w}_j$.
		\item \textit{Phase II: Decoding.} After receiving 1-bit signal $\mathbf{c}_j$ and length $\left\|\mathbf{w}_j\right\|$, node $i$ tries to decode $\mathbf{w}_j$ as presented in Algorithm \ref{algorithm-1bcs}. This is the well-known 1BCS \cite{boufounos20081}. The decoding process (i.e., solving problem (\ref{1bcs})) can be accomplished using a gradient projection subspace pursuit (GPSP) algorithm proposed in \cite{zhou2022computing}. It has been shown in \cite{zhou2022computing}, \cite{jacques2013robust} that solution $\mathbf{w}_j$ can be decoded successfully with high probability if $\mathbf{\Phi}_i$ is a randomly generated Gaussian matrix and $d_i$ exceeds a certain threshold. Moreover, both theoretical and numerical results have demonstrated that the sparser (i.e., the smaller $s$) solution $\mathbf{w}_j$, the easier the problem to be solved. This explains why we introduce a sparsity constraint in (\ref{objfunc}).
	\end{itemize}

	\begin{algorithm}
		\caption{$\mathbf{z}_j=1 \operatorname{BCS}\left(\mathbf{w}_j, \mathbf{\Phi}_i, \gamma_j\right)$}
		\textit{$--$ encoding by node $j--$} \\
		1. Compute $\left\|\mathbf{w}_j\right\|$. \\
		2. Let $\mathbf{x}_j=\operatorname{sign}\left(\mathbf{w}_j\right) \odot \log _{\gamma_j}\left(1+\left|\mathbf{w}_j\right|\right).$ \\
		3. Compress $\mathbf{x}_j$ by $\mathbf{c}_j:=\operatorname{sign}\left(\mathbf{\Phi}_i\left(\mathbf{x}_j /\left\|\mathbf{x}_j\right\|\right)\right)$. \\
		4. Send $\left(\left\|\mathbf{w}_j\right\|, \mathbf{c}_j\right)$ to node $i$. \\
		\textit{$--$ decoding by node $i--$} \\
		5. Find a solution $\mathbf{v}_j$ to 
		\begin{equation}
			\mathbf{c}_j=\operatorname{sign}\left(\mathbf{\Phi}_i \mathbf{v}\right), \quad\|\mathbf{v}\|_0 \leq s.
			\label{1bcs}
		\end{equation}
		
		\noindent6. Compute $\mathbf{v}_j=\operatorname{sign}\left(\mathbf{v}_j\right) \odot\left(\gamma_j^{\left|\mathbf{v}_j\right|}-1\right)$. \\
		7. Let $\mathbf{z}_j=\left(\left\|\mathbf{w}_j\right\| /\left\|\mathbf{v}_j\right\|\right) \mathbf{v}_j$ be an estimator to $\mathbf{w}_j$.
		\label{algorithm-1bcs}
	\end{algorithm}

	\subsection{Communication Efficiency}\label{Com-eff}
	Our approach enhances communication efficiency through two key factors. Firstly, communication among neighbouring nodes ${N}_i$ only occurs when ${k \in \mathcal{K}_i=\left\{\kappa_i, 2 \kappa_i, \cdots\right\}}$, meaning it does not happen at every iteration step. Therefore, a larger $\kappa_i$ allows node $i$ more steps to update its local parameters, leading to an improved communication efficiency. Such a scheme has been extensively employed in literature \cite{zheng2016asynchronous,yu2019parallel,wang2021cooperative,mcmahan2017communication,li2019convergence,li2020federated,zhou2023fedgia}. Most importantly, as shown in Algorithm \ref{algorithm-1bcs}, only 1-bit signal $\mathbf{c}_j \in \mathbb{R}^{d_i}$ and a positive scalar are transmitted among neighbour nodes, significantly enhancing communication efficiency. Moreover, if we take $d_i<n$, then the signal is a compressed version of $\mathbf{w}_j \in \mathbb{R}^n$, further reducing communication costs. Taking ${d_i=n=10^4}$ as an example, transmitting a dense vector $\mathrm{w} \in \mathbb{R}^n$ with double-precision floating entries directly between two neighbours would need $64 \times 10^4$ bits. In contrast, Algorithm \ref{algorithm-1bcs} only transmits $\left(64+10^4\right)$-bit content.
	
	\subsection{Computational Efficiency}
	While the decoding process in Algorithm \ref{algorithm-1bcs} for node $i$ may incur some computational cost, it only occurs at step ${k \in \mathcal{K}_i}$. From \cite{zhou2022computing}, the decoding complexity is about $O\left(c_1\left(c_2 d_i n+d_i s^2+s^3\right)\right)$, where $c_1$ and $c_2$ are two small integers. In general, $s^2 \leq \min \left\{d_i^2, n\right\}$, thereby the complexity reducing to $O\left(c_1 c_2 d_i n\right)$. If we generate sparse $\boldsymbol{\Phi}_i$ with $\alpha$ (e.g. $0.05)$ percent of entries to be non-zeros for each $i \in {V}$, then the complexity is further reduced to $O\left(\alpha c_1 c_2 d_i n\right)$. Moreover, as the communication only occurs at $k \in \mathcal{K}_i$, complex items $\mathbf{u}_i^{k}$ in (\ref{def-w-bar-tilde})  only need to be computed once for consecutive $\kappa_i$ iterations. While the computation for $\mathbf{w}_i^{k}$ is quite cheap, with computational complexity about $O\left(n m_i^{k}+s \log (s)\right)$.
	
	\subsection{Partial Device Participation and Partial Synchronization}
	Each node $i$ only selects a subset of neighbour nodes, ${N}_i^{k} \subseteq {N}_i$, for the training at every step, enabling node $i$ to deal with the straggler's effect. Following the approach in \cite{li2019convergence}, node $i$ sets a threshold $t_i \in [1, |{N}_i|)$ and collects signals from the first $t_i$ responding neighbour nodes to form ${N}_i^{k}$. Once $t_i$ signals are received, node $i$ proceeds without waiting for the remaining nodes, which are considered stragglers for that iteration. In practical applications, a neighbour node $j$ with unreliable and imperfect communication links can be treated as a straggler, indicating that node $i$ should not pick it, i.e., $j \notin {N}_i^{k}$. Note that such a strategy is known for partial device participation in centralized FL \cite{li2020secure,zhou2023federated,li2022federated}.
	
	Moreover, one can observe that Algorithm $\ref{algorithm-{CEPS}}$ enables node $i$ to keep using $\mathbf{z}_j^{k}$ for nodes $j \notin {N}_i^{k}$ as these parameters are previously stored at node $i$. Therefore, node $i$ only needs to wait for signals from nodes ${N}_i^{k}$ to synchronize.

		Finally, each node $i$ operates with a communication interval $\kappa_i$, allowing for asynchronous behavior across different nodes. This flexibility enables each node to control its local training schedule independently.
	
	\subsection{Privacy Guarantee}
	To ensure privacy protection, we apply the Gaussian mechanism to the gradient computations, adding random noise to achieve differential privacy. Since the raw dataset at each node is accessed only during the gradient calculation, adding noise at this step ensures that the entire algorithm adheres to differential privacy through the post-processing property \cite{abadi2016deep}. Specifically, in each iteration, the computation at node $i$ can be viewed as $h \circ \nabla f_i$, where $\nabla f_i: {D}_i \rightarrow \mathbb{R}^n$ is the gradient step involving the data, and $h: \mathbb{R}^n \rightarrow \mathbb{R}^n$ represents subsequent computations in (\ref{def-w-bar-tilde}) that do not access the data.
	
	We emphasize that this privacy-preserving mechanism not only provides rigorous differential privacy guarantees but also enhances robustness. Robustness, in this context, refers to the model's ability to maintain consistent predictions in the presence of small adversarial perturbations $\boldsymbol{\delta}$ to input data $\mathbf{x}$. 
As discussed in \cite{lecuyer2019certified, asi2023robustness}, the noise magnitude required for achieving a certain level of privacy (characterized by $(\epsilon, \delta)$) can be leveraged to provide robustness guarantees.

	\section{Theoretical Analysis}
	\label{Theory}
	In this section, we provide a theoretical analysis of our proposed algorithm, including the assumptions, matrix representations, convergence analysis, and privacy guarantees.
	\subsection{Convergence analysis}
	To simplify the convergence analysis, we choose the following parameters.
	\begin{itemize}[leftmargin=13pt]
		\item[1)] Set ${\W^0 = \mathbf{0}}$ and let ${\kappa_i= k_0}$ for any ${i\in{V}}$ for simplicity (In fact, we can let $k_0$ be the least common multiple of $\{\kappa_1,\kappa_2,\ldots,\kappa_m\}$ and then the subsequent analysis remains similar to the case of ${\kappa_i= k_0}$). This allows us to focus on steps $ak_0, a=0,1,2,3,\ldots$. Therefore, at $a k_0$th step, all nodes communicate with their neighbours. At this moment, we define an average point by
		\begin{equation}\label{avgpoint}
			\begin{aligned}
				\boldsymbol{\varpi}^{a k_0}  :=\frac{1}{m}\sum_{i=1}^m\mathbf{w}_i^{a k_0}.
			\end{aligned}
	\end{equation}
		\item[2)] Set $m_i = m_j, \sigma_i = {c}/{ t_i}, 	m_i^k = t_i $ for $\forall~ i,j, k$, where  $c>0$ can be chosen flexibly and $t_i\leq m_i$ is a positive integer.
	\end{itemize}
	To build the convergence, we make the following assumptions.
	\begin{assumption}\label{assump-gradientlip}
		For every $ {i \in {V}}$, gradient $\nabla f_i$ is  Lipschitz continuous with a  constant ${\ell_i>0}$. Denote $\ell:=\max_{i\in{V} }\ell_i$.
	\end{assumption}
	\begin{assumption}\label{assump-variancebound}
	Let $\mathcal{U}({V})$ be a uniform distribution on ${V}$. For any $\w$,	${\mathbb{E}_{i \sim \mathcal{U}({V})}\|\nabla f_i(\mathbf{w})- \frac{1}{m}\sum_{j\in{V}} \nabla f_j(\mathbf{w})\|^2 \leqslant \zeta^2}$.		
	\end{assumption}
	\begin{assumption}\label{assump-1bcs}
		1BCS technique accurately recovers $\mathbf{w}_j^k$.	
	\end{assumption}
Noted that Assumptions \labelcref{assump-gradientlip,assump-variancebound,} are standard in DFL, and  Assumption \ref{assump-1bcs} can be justified with high probability  \cite{zhou2022computing,jacques2013robust}.
	\begin{assumption}\label{assump-communicate} Let {${E}^k:=\{(i,j):j\in{N}_i^k,i\in{V}\}$} be the set of the selected edges at step $k$ and $B$ be a given integer. Suppose graph
	\begin{equation}\label{connection-G}
	{G}_t^B:=\left({V},{E}^{tk_0}\cup{ E}^{(t+1)k_0}\cup\ldots\cup{ E}^{(t+B)k_0} \right)
		\end{equation}
is connected for any $t\in\{0,1,2,\ldots\}$.
		\end{assumption} 
		The above assumption is related to the  uniformly
strong connection \cite{nedic2014distributed} and can be satisfied with a high probability. For example, a sufficient condition to ensure this assumption is that every pair of clients communicates at least once for each $B$ communication rounds, and one can check that this sufficient condition  holds with a probability 
		$$
		p \geq \prod_{i=1}^m \left[1 - \left(1 - \frac{t_i - 1}{m - 1}\right)^B \right]^{m - 1},
		$$	
		which is relatively high, e.g., ${p\approx 1}$ when ${m = 32}$, ${B\geq 20}$,  and ${t_i \geq  ({m+1})/2}$. Assumption \ref{assump-communicate} allows us to avoid imposing specific conditions on the communication matrix, thereby increasing its generality and extending its applicability. Finally, letting ${B>0}$  given in Assumption \ref{assump-communicate}, we define
			\begin{equation}\label{lower-bd-c}
			\begin{aligned}
			c_0 := \left(\frac{ 80m\sqrt{\ell} }{1-\tau}\right)^2,~ \tau :=\left(1-\frac{1}{m^{mB}}\right)^{\frac{1}{B}}.  
			\end{aligned}
		\end{equation}		
It is clear that ${\tau\in(0,1)}$. When $\mathbf{w}_1=\cdots=\mathbf{w}_m$, we write 
\begin{equation*} 
 f(\w):= \sum_{i \in {V}} f_i(\w). 
	\end{equation*} 
	\begin{theorem}		\label{theorem2}
		Under Assumptions \labelcref{assump-gradientlip,assump-variancebound,assump-1bcs,assump-communicate} and setting ${c\geq c_0}$, the following bound holds for Algorithm \ref{algorithm-{CEPS}},
			\begin{equation*}
			\begin{aligned}
					 \frac{1}{  T} \sum_{a=0}^{T-1} \mathbb{E}\left\|\nabla f\left(\boldsymbol{\varpi}^{ak_0}\right)\right\|^2 
				 {\leq}~&~  \frac{8mc(f\left(0\right) - f ^*)}{T}\\
				 + ~& ~ \frac{m^2c_0\zeta^2 + 16mck_0  e_{\infty}}{c}, 
			\end{aligned}
		\end{equation*}
		where $\varrho :=\sum_{i=1}^m\varrho_i$ and $e_{\infty}$ is an error bound relied on the sparse projection of each iteration and the added noise. 
	\end{theorem}
    \begin{proof}
        The detailed proof is presented in Appendix A-C (available in supplemental material).
    \end{proof}
In the above theorem, the upper bound of the error consists of three components. The second error term, ${c_0m^2\zeta^2}/{c}$, arises from sampling to select neighbour nodes. The third error term, $16m k_0 e_{\infty}$, is due to the sparse projection and the added noise for privacy preservation. The detailed formulation of $e_{\infty}$ is provided in Lemma 2 in the Appendix.
	\subsection{Privacy Guarantee}
	We now demonstrate that our algorithm provides differential privacy guarantees at each iteration as well as the entire process for datasets ${{D}_1, {D}_2, \dots, {D}_m}$.  Hereafter, we set 
	\begin{equation}\label{variance-Gaussian}
		\varrho_i:=\frac{2 \ln (1.25 / \delta) u_{i }^2}{\varepsilon^2},~~ i\in\mathcal{V},
	\end{equation}
	and consider the algorithm as a query process ${\mathcal{A}: {D} \rightarrow \mathbb{R}^n}$, where ${{D} = ({D}_1, {D}_2, \dots, {D}_m)}$ represents the collection of all local datasets, and $D_i$ is the data of node $i$. Moreover, we need the following condition that has been frequently assumed in differential privacy analysis \cite{pappas2021ipls,koloskova2019decentralized,liu2022decentralized}.
		\begin{assumption}\label{assump-gradientbound}
		 ${\|\nabla f_i (\mathbf{w})\| \leq u_i/2}$ for any $\mathbf{w}$ and any $ {i \in {V}}$.  
	\end{assumption}
	Here we point out that the existing literature considering DP in decentralized FL only gives a privacy guarantee in each iteration, such as \cite{li2020secure}. Instead, we give the privacy guarantee by viewing the algorithm as a whole part, which is more realistic for various applications.

	\begin{theorem}\label{privacy-each-node}
		Given ${\varepsilon>0, \delta \in (0,1)}$, each step of Algorithm \ref{algorithm-{CEPS}} achieves $(\varepsilon,\delta)$-differential privacy under Assumption \ref{assump-gradientbound}.
	\end{theorem}
	\begin{proof}
        The detailed proof is presented in Appendix D (available in supplemental material).
    \end{proof}
	
	\begin{theorem}\label{privacy-all}  Given ${\varepsilon>0, \delta \in (0,1)}$, Algorithm \ref{algorithm-{CEPS}} with $ak_0$ total iterations achieves $$\left(\sqrt{2 a \ln\left(1 / \delta\right)} \varepsilon+  a \varepsilon\left(e^{\varepsilon}-1\right),(a+1) \delta\right)-$$
		differentially privacy under Assumption \ref{assump-gradientbound}.
	\end{theorem}
	\begin{proof}
        The detailed proof is presented in Appendix E (available in supplemental material).
    \end{proof}

	\section{Numerical Experiments}
	\label{Numerical}
	In this section, we present numerical experiments
	to evaluate the performance of {CEPS}. All experiments are implemented using MATLAB (R2021b) on a laptop with 32GB memory and 3.7GHz CPU.
	
	\subsection{Testing Example}
	\begin{example}{(Linear Regression)}.
		Consider a linear regression problem where each node ${i \in {V}}$ has an objective function
		$$f_i(\mathbf{w})=\frac{1}{2 m_i}\left\|\mathbf{A}_i \mathbf{w}-\mathbf{b}_i\right\|^2,$$
		where ${\mathbf{b}_i \in \mathbb{R}^{m_i}}$ and ${\mathbf{A}_i \in \mathbb{R}^{m_i \times n}}$ is a random Gaussian matrix. To assess the effectiveness of {CEPS}, we assume there is a `ground truth' solution ${\mathrm{w}^* \in \mathbb{R}^n}$ with $s$ non-zeros entries. Then ${\mathbf{b}_i:=\mathbf{A}_i \mathbf{w}^*+0.5 \mathbf{e}_i}$, where $\mathbf{e}_i$ is the noise. All entries of $\mathbf{A}_i$ and $\mathbf{e}_i$ are identically and independently distributed from a standard normal distribution, while the non-zero entries of $\mathrm{w}^*$ are uniformly generated from $[0.5,2] \cup[-2,-0.5]$.
		\label{eg5.1}
	\end{example}
	\begin{example}{(Logistic Regression).}
		Consider a logistic regression problem where each node $i$ has an objective function
		$$
		f_i(\mathbf{w})=\frac{1}{m_i} \sum_{t=1}^{m_i}\left(\ln \left(1+\mathrm{e}^{\left\langle\mathbf{a}_i^t, \mathbf{w}\right\rangle}\right)-b_i^t\left\langle\mathbf{a}_i^t, \mathbf{w}\right\rangle\right)+\frac{\lambda}{2}\|\mathbf{w}\|^2,
		$$
		where ${\mathbf{a}_i^t \in \mathbb{R}^n}, {b_i^t \in\{0,1\}}$, and ${\lambda>0}$ is a penalty parameter (e.g., ${\lambda=0.001}$ in our numerical experiments). We use two real datasets: `qot' (Qsar oral toxicity) from UCI \cite{asuncion2007uci} and `rls' (real-sim) from Libsvm \cite{chang2011libsvm} to generate $\mathbf{a}_i$ and $b_i$. The dimensions of two datasets are ${(m,n)=(1024, 8992)}$ and ${(m,n)=(20958, 72309)}$. The samples are randomly divided into $m$ groups, corresponding to $m$ nodes.
		\label{eg5.2}
	\end{example}

	\subsection{Implementations}
	The basic setup for {CEPS} in Algorithm \ref{algorithm-{CEPS}} is given as follows. We begin by randomly generating a graph and obtain ${\left\{{N}_i: i \in {V}\right\}}$. For each node ${i \in {V}}$, we randomly select subset ${N}_i^{k}$ such that ${t_i=\left|{N}_i^{k}\right|=r\left|{N}_i\right|}$, where ${r \in(0,1]}$ is the participation rate. Encoding matrix ${\boldsymbol{\Phi}_i \in \mathbb{R}^{d_i \times n}}$ is generated similarly to $\mathbf{A}_i$ with ${d_i=n/2}$. Other parameters are initialized as ${\gamma_i=5}$, ${\mu=0.1}$, ${\sigma_i=\lambda_{\max }\left(\mathbf{A}_i^{\top} \mathbf{A}_i\right) /\left(m(2r+0.1) d_i\right)}$, where ${\lambda_{\max }(\mathbf{A})}$ stands for the largest eigenvalue of $\mathbf{A}$, and $\kappa_i$ is randomly chosen from ${[10,15]}$. The privacy constants in \eqref{variance-Gaussian} are set as ${\delta=0.5}$, ${u_{i }=0.1}$, and ${\varepsilon\in[0.1,1]}$. 
	Finally, we terminate the algorithm if 
	$$\frac{1}{sm}\Big\|\W^k-(\boldsymbol{\varpi}^{k},\cdots,\boldsymbol{\varpi}^{k})\Big\|^2\leq  {\rm tol},$$
	where ${\rm tol}=0.005$ if no DP techniques are used in {CEPS} (i.e., $\varrho_i=0$) and ${\rm tol}=0.0025/\varepsilon$ otherwise. 
	
	\subsection{Self-comparison of {CEPS}}
	To assess the performance of the 1BCS and DP techniques in Algorithm \ref{algorithm-1bcs}, we compare four variants of {CEPS} with and without using these two techniques. As shown in Table \ref{self-{CEPS}},  `Perf' stands for perfect communication (i.e., without 1BCS).  That is, ${\mathbf{z}_j^{k}=1 \mathrm{BCS}\left(\mathbf{w}_j^k, \boldsymbol{\Phi}_i, \gamma_j\right)}$ is replaced by ${\mathbf{z}_j^{k}=\mathbf{w}_j^k}$ for every ${j \in {N}_i^{k}}$ in Algorithm \ref{algorithm-1bcs}. `No DP' means no noise added in \eqref{def-w-bar-tilde}, namely $\varrho_i=0$. Therefore, `DP-Perf' represents {CEPS} with DP but without 1BCS, and `NoDP-Perf' represents {CEPS} without DP and 1BCS. For Example \ref{eg5.1}, we set ${n=1000}$ and ${s=10}$ and generate each ${m_i \in[250,750]}$, ${i \in {V}}$ to see the impact of node  number $m$, noisy ratio $\varepsilon$, and participation rate $r$ on the convergence performance of {CEPS}. 
	
		\begin{table}[!t]
		\renewcommand{\arraystretch}{1.25}\addtolength{\tabcolsep}{3pt}
		\caption{Four variants of {CEPS}.}\vspace{-2mm}
		\begin{center}
			\begin{tabular}{c|cc}
				\hline 
				 &1BCS& No 1BCS\\\hline
 			 DP& DP-1BCS 	& DP-Perf \\  
			    No DP &NoDP-1BCS 	& NoDP-Perf \\
				\hline
			\end{tabular} 
			\label{self-{CEPS}}
		\end{center}
	\end{table}

 	\begin{table}[!t]
		\renewcommand{\arraystretch}{1.25}\addtolength{\tabcolsep}{-4.2pt}
		\caption{Self-Comparison of {CEPS}.}\vspace{-2mm}
		\begin{center}
			\begin{tabular}{lcccccccccccc}
				\hline
				&&\multicolumn{3}{c}{$m$} &&\multicolumn{3}{c}{$\varepsilon$}&&\multicolumn{3}{c}{$r$}\\\cline{3-5} \cline{7-9} \cline{11-13} 				
	&&	32 	&	64 	&	128 	&&	0.25 	&	0.50 	&	0.75 	&&	0.20 	&	0.50 	&	0.80 	\\\hline
	&&	\multicolumn{11}{c}{Objective}\\\hline																	
DP-1BCS	&&	0.127 	&	0.125 	&	0.125 	&&	0.129 	&	0.126 	&	0.126 	&&	0.127 	&	0.127 	&	0.129 	\\
NoDP-1BCS	&&	0.126 	&	0.125 	&	0.125 	&&	0.127 	&	0.126 	&	0.126 	&&	0.126 	&	0.127 	&	0.130 	\\
DP-Perf	&&	0.127 	&	0.125 	&	0.125 	&&	0.129 	&	0.126 	&	0.126 	&&	0.126 	&	0.130 	&	0.131 	\\
NoDP-Perf	&&	0.126 	&	0.125 	&	0.125 	&&	0.127 	&	0.125 	&	0.125 	&&	0.126 	&	0.126 	&	0.133 	\\\hline
	&&	\multicolumn{11}{c}{Time (seconds)}\\\hline																	
DP-1BCS	&&	1.225 	&	3.689 	&	14.20 	&&	1.101 	&	1.322 	&	1.432 	&&	1.224 	&	2.412 	&	2.713 	\\
NoDP-1BCS	&&	1.105 	&	3.551 	&	13.72 	&&	1.095 	&	1.171 	&	1.179 	&&	1.086 	&	2.292 	&	2.607 	\\
DP-Perf	&&	0.269 	&	0.587 	&	1.211 	&&	0.290 	&	0.287 	&	0.300 	&&	0.282 	&	0.290 	&	0.287 	\\
NoDP-Perf	&&	0.286 	&	0.589 	&	1.195 	&&	0.283 	&	0.294 	&	0.304 	&&	0.295 	&	0.291 	&	0.291 	\\\hline
	&&	\multicolumn{11}{c}{Iterations}\\\hline																	
DP-1BCS	&&	32	&	30	&	30	&&	30	&	32	&	35	&&	32	&	30	&	26	\\
NoDP-1BCS	&&	29	&	29	&	29	&&	29	&	29	&	29	&&	29	&	29	&	23	\\
DP-Perf	&&	28	&	28	&	28	&&	28	&	30	&	32	&&	28	&	28	&	24	\\
NoDP-Perf	&&	27	&	25	&	27	&&	27	&	27	&	27	&&	27	&	27	&	21	\\\hline
			\end{tabular}
			\label{tab1}
		\end{center}
	\end{table}
 
 	\begin{figure*}[t]
		\centering
			\includegraphics[width=.325\textwidth]{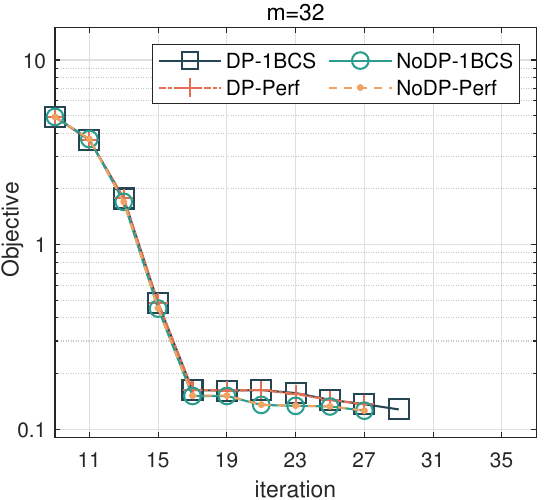}
			\includegraphics[width=.325\textwidth]{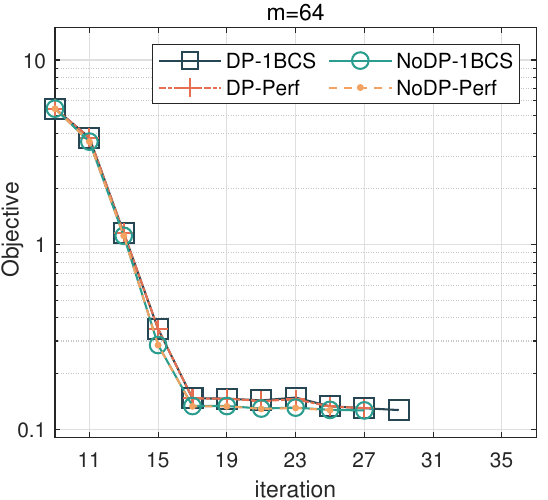}
			\includegraphics[width=.325\textwidth]{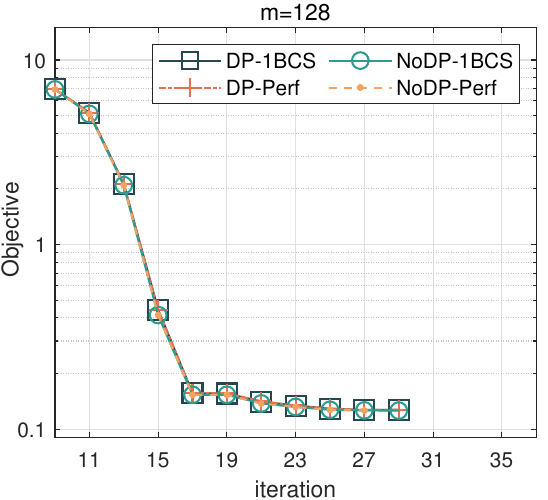}
		\caption{Self-comparison of {CEPS}: Effect of $m$.}
		\label{{CEPS}-m}
	\end{figure*}
		\begin{figure*}[t]
		\centering
			\includegraphics[width=.325\textwidth]{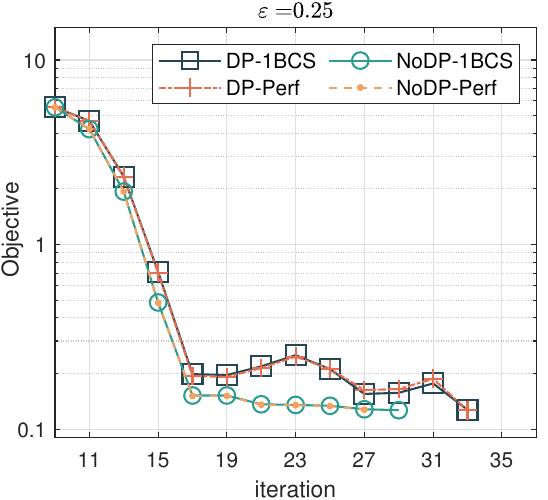}
			\includegraphics[width=.325\textwidth]{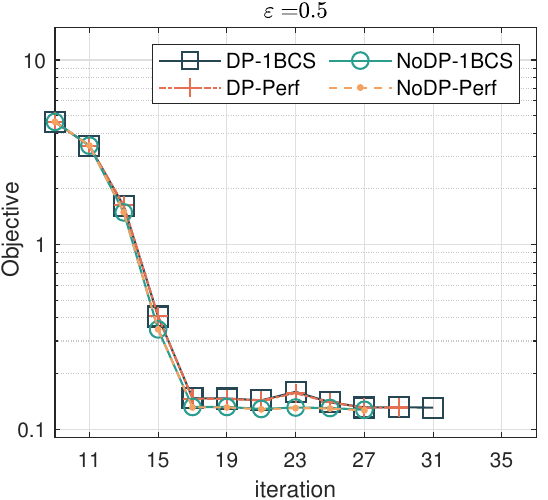}
			\includegraphics[width=.325\textwidth]{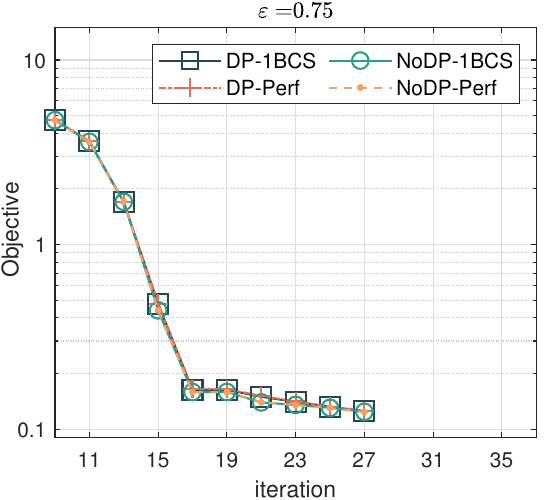}
		\caption{Self-comparison of {CEPS}: Effect of $\varepsilon$.}
		\label{{CEPS}-vareps}
	\end{figure*} 

		\begin{figure*}[t]
		\centering
			\includegraphics[width=.325\textwidth]{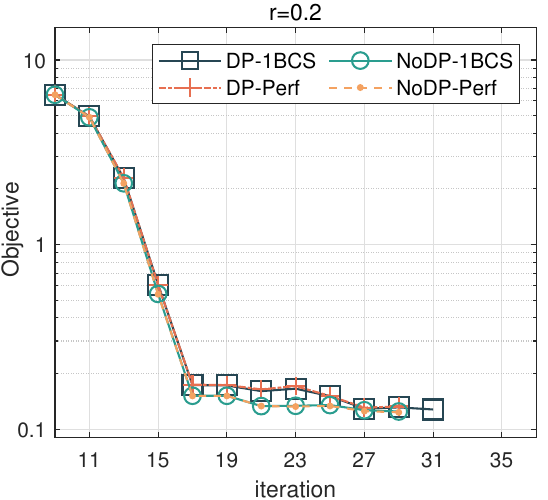}
			\includegraphics[width=.325\textwidth]{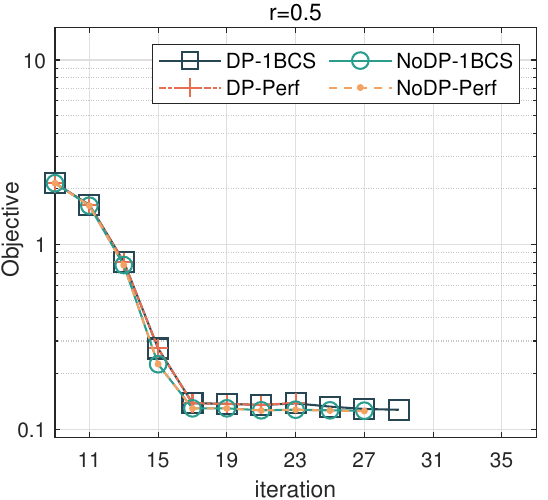}
			\includegraphics[width=.325\textwidth]{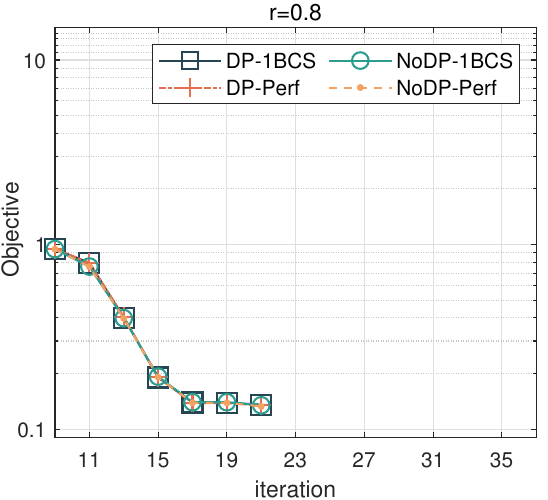}
		\caption{Self-comparison of {CEPS}: Effect of $r$.}
		\label{{CEPS}-r}
	\end{figure*} 
	
{\it 1) Effect of $m$.} Fig. \ref{{CEPS}-m} illustrates the effect of varying $m\in$ ${\{32,64,128\}}$ by fixing ${(\varepsilon,r)=(0.5,0.2)}$. For fewer nodes (e.g., ${m=32}$), DP-1BCS requires more iterations to achieve the optimal function values. As the number of nodes increases, the performance differences among the four {CEPS} variations diminish. From Table \ref{tab1}, no significant differences are observed in objective values or iteration counts among the four algorithms. However, it is evident that as $m$ increases, the computational time also rises.
 	
{\it 2) Effect of $\varepsilon$.} Fig. \ref{{CEPS}-vareps} demonstrates the effect of altering ${\varepsilon\in\{0.25,0.5,0.75\}}$ by fixing ${(m,r)=(32,0.2)}$. According to \eqref{def-w-bar-tilde}, smaller $\varepsilon$ introduces more noise to perturb the transmitted messages. This enhances privacy but also results in noticeable fluctuations during the training process, as observed in the figure. For ${\varepsilon = 0.25}$, NoDP-1BCS and NoDP-Perf exhibit much smoother behavior compared to DP-1BCS and DP-Perf. As $\varepsilon$ increases (e.g., ${\varepsilon \geq 0.5}$), the objective function values of all four algorithms follow a similar declining trend, indicating that the impact of noise becomes negligible. From Table \ref{tab1}, no significant differences are observed in objective values. However, more iterations and time are required for DP-1BCS and DP-Perf when $\varepsilon$ decreases. 

{\it 3) Effect of $r$.} Fig. \ref{{CEPS}-r} presents the effect of changing ${r\in\{0.2,0.5,0.8\}}$ by fixing ${(m,\varepsilon)=(32,0.5)}$. It appears that the range of ${r \in [0.2, 0.8]}$ has minimal impact on the convergence of these four variants, as the produced curves closely coincide.  From Table \ref{tab1}, their generated objective function values are comparable. However, as more neighbour modes are selected (i.e., as $r$ increases), the computational time increases while the number of iterations decreases.

	\subsection{Comparison with other federated learning methods}
We evaluate the performance of {CEPS} under two communication scenarios: ideal communication ({CEPS}-Perf) and communication using 1BCS ({CEPS}-1BCS) by comparing them with D-PSGD \cite{lian2017can}, DFedAvgM \cite{sun2022decentralized}, and DFedSAM \cite{shi2023improving}. Additionally, we also introduce two variants of D-PSGD: D-PSGD-DN and D-PSGD-PC. The former uses a dynamic mixing matrix network and the latter employs partial communication. To ensure fair comparison, all algorithms perform communication among nodes every 10 iterations. Finally, for D-PSGD and D-PSGD-DN, each node imposes all neighbours to join in training, while for the other five methods, each node selects $r|N_i|$ neighbour nodes to participate. We now assess the performance of these methods on  Examples \ref{eg5.1} and \ref{eg5.2}. 

    {\it 4) Effect of $m$.} 
	 Fig. \ref{ex_convergence} illustrates the convergence behavior for Example \ref{eg5.1} as the number of clients, $m$, varies over ${\{32,64,128\}}$. Generally, increasing the number of communication nodes reduces the required communication rounds. Notably, {CEPS} consistently requires the fewest communication rounds to converge across all scenarios. Plain D-PSGD converges second fastest, benefiting from the participation of all clients. Incorporating a dynamic network in D-PSGD-DN has only a minor impact on convergence speed. In contrast, partial communication in D-PSGD-PC results in slower convergence. While both DFedAvgM and DFedSAM outperform D-PSGD-PC, neither matches the performance of {CEPS}, which achieves the fastest convergence overall. 

	Figs. \ref{example1_cr}-\ref{example1_time} provide additional numerical performance metrics, including communication rounds (CR), data transmission volume (DTV), and computational time. The left column shows results for Example \ref{eg5.1}, the middle column corresponds to Example \ref{eg5.2} using the `qot' dataset, and the right column presents results for the `rls' dataset.
Fig. \ref{example1_cr} presents the expected trend: increasing  $m$ generally leads to fewer CR. Once again, {CEPS} outperforms all other methods across all scenarios. Fig. \ref{example1_data} highlights the communication efficiency of {CEPS}-1BCS during training. Compared to all other methods, it achieves a substantial reduction in data transmission volume, saving at least 90$\%$ of the communication overhead. Even without 1BCS, {CEPS}-Perf still outperforms the other methods in terms of DTV. Fig. \ref{example1_time} compares the computational time across the seven methods. {CEPS}-1BCS achieves faster computation compared to DFedAvgM and DFedSAM, although its computational time is slightly higher than that of D-PSGD. However, it is important to note that the reported time does not include communication time. In practical applications, communication time is typically much more expensive than the computational time. As a result, {CEPS}-1BCS remains highly competitive, as evidenced by the reduced CR shown in Fig. \ref{example1_cr}. Additionally, {CEPS}-Perf consistently consumes the shortest computational time across all scenarios. 

    {\it 5) Effect of $\varepsilon$.}  
	To see such an effect, we employ all methods to solve Example \ref{eg5.1} by fixing ${(m,n,r) = (64,1000,0.2)}$ but altering ${\varepsilon \in \{0.5,1,2\}}$. The results are presented in Table \ref{tabnoise}. One can observe that when $\varepsilon$ is small (i.e. the privacy budget is low), all the methods are impacted by the noise. Specifically, all D-PSGD variants fail to converge when ${\varepsilon = 0.5}$. However, as $\varepsilon$ increases, all methods show improved performance. Evidently,  {CEPS} yields the best results across all settings.
	
	     \begin{table}[!t]
		\renewcommand{\arraystretch}{1.25}\addtolength{\tabcolsep}{-2.5pt}
		\caption{Comparison of $\varepsilon$.}\vspace{-3mm}
		\begin{center}
			\begin{tabular}{lccccccccc}
				\hline
				& \multicolumn{3}{c}{$\varepsilon=0.5$} & \multicolumn{3}{c}{$\varepsilon=1$} & \multicolumn{3}{c}{$\varepsilon=2$}\\
				\cline{2-10} 
				& Obj& Iter & Time & Obj& Iter & Time & Obj& Iter & Time \\
				\hline
				CEPS-Perf & 0.124& 24 & 1.038 & 0.123& 22 &1.062 & 0.123 & 22 &1.026  \\
				CEPS-1BCS & 0.123& 24 & 2.997 &0.123 &22 & 3.111 & 0.123 & 22 & 2.161 \\
			DPSGD &--*&--&-- &  
                0.561 &  41 & 2.692& 0.129 & 81 & 3.872 \\
                DPSGD-DN &--&--&-- & 0.192 & 71 & 3.509& 0.129 & 91 & 4.114 \\
				DPSGD-PC &--&--&-- & 0.185 & 81 & 3.805& 0.132 & 101 & 4.492 \\
                DFedAvgM &0.577 &51 & 2.911 & 0.274 & 61 & 3.289& 0.156& 91 & 4.199 \\
                DFedSAM &0.636 &41 & 3.221 & 0.243 & 61 & 4.049& 0.150 & 81 & 4.914 \\
				\hline
			\end{tabular}
					\begin{tablenotes}
				* `--' means that the algorithm diverges.
			\end{tablenotes}
			\label{tabnoise}
		\end{center}
	\end{table}
	
    {\it 6) Effect of $r$.}  
		Again for Example \ref{eg5.1}, we fix ${(m,n,\varepsilon)} = {(64,1000,2)}$ but vary ${r \in \{0.2, 0.5, 0.8\}}$. The results are reported in Table \ref{tabr}. Generally, increasing $r$ leads to fewer iterations but higher  DTV. This is expected, as a higher participation rate incorporates more clients into the communication network, thereby increasing communication overhead. As always, {CEPS} stands out among all methods, maintaining relative stability while consistently achieving the highest communication efficiency.
	
	 \begin{table}[!t]
		\renewcommand{\arraystretch}{1.25}\addtolength{\tabcolsep}{-2.6pt}
		\caption{Comparison of $r$.}\vspace{-3mm}
		\begin{center}
			\begin{tabular}{lccccccccc}
				\hline
				& \multicolumn{3}{c}{$r = 0.2$} & \multicolumn{3}{c}{$r = 0.5$} & \multicolumn{3}{c}{$r = 0.8$}\\
				\cline{2-10} 
				& Obj& Iter & DTV* & Obj& Iter & DTV & Obj& Iter & DTV \\
				\hline
				CEPS-1BCS & 0.126& 38 & 0.78 & 0.126& 28 &1.36 & 0.126 & 26 &2.05  \\
				CEPS-Perf & 0.126& 28 & 3.65 &0.126 &28 & 9.13 & 0.126 & 24 & 12.52 \\
				NoDP-1BCS & 0.126 &29 & 0.63 & 0.126& 29 & 1.44 & 0.127 &23 & 1.95 \\
				NoDP-Perf & 0.125 &29 & 3.78 & 0.125 & 27 & 8.80 & 0.125 &21 & 10.96  \\
				DPSGD-PC &0.131 &81 & 13.27 & 0.128 & 61 & 24.99& 0.128 & 51 & 33.42 \\
				DFedAvgM &0.140 &61 & 9.99 & 0.128 & 51 & 20.89& 0.133& 31 & 20.32 \\
				DFedSAM &0.151 &61 & 9.99 & 0.166 & 51 & 20.89& 0.208 & 31 & 20.32 \\
				\hline
			\end{tabular}
			\begin{tablenotes}
				*The unit of DTV is $10^6$ bytes.
			\end{tablenotes}
			\label{tabr}
		\end{center}
	\end{table}

        \begin{figure*}[t]
		\centering
			\includegraphics[width=.325\textwidth]{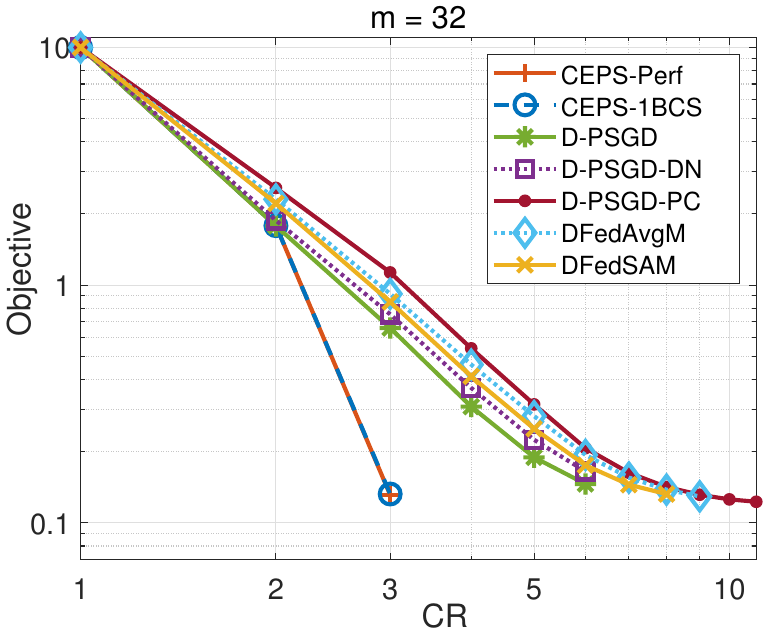}
				\includegraphics[width=.325\textwidth]{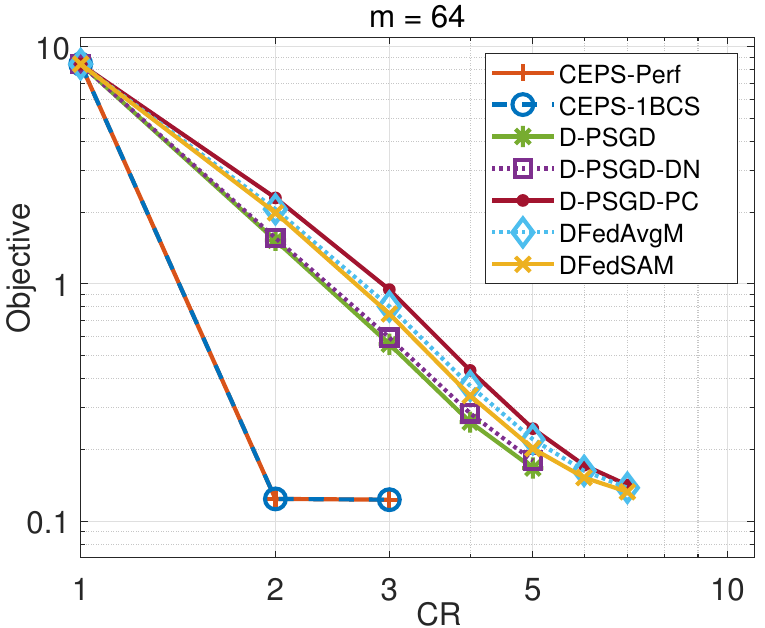}
					\includegraphics[width=.325\textwidth]{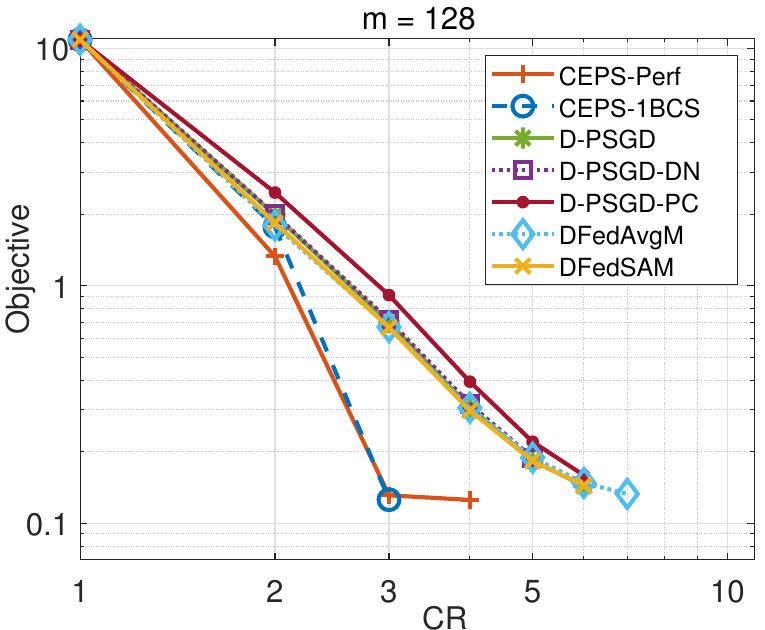}
		\caption{Objective v.s. CR.}
		\label{ex_convergence}
	\end{figure*}

	\begin{figure*}[!t]
		\centering 
			\includegraphics[width=0.325\textwidth]{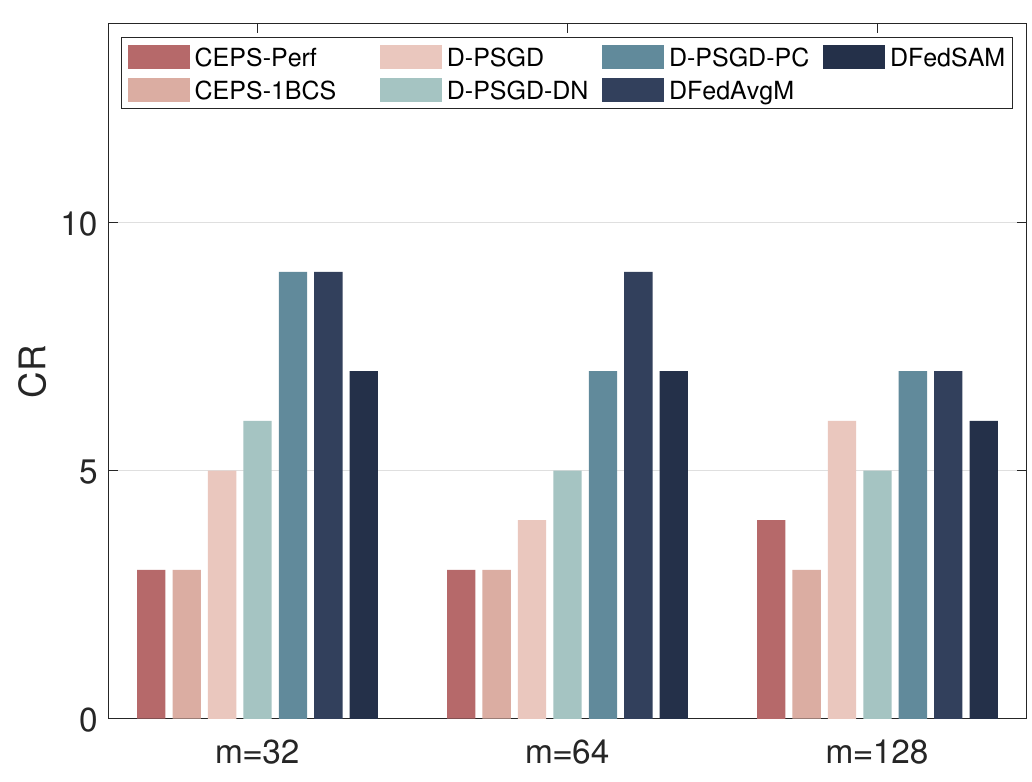}
			\includegraphics[width=0.325\textwidth]{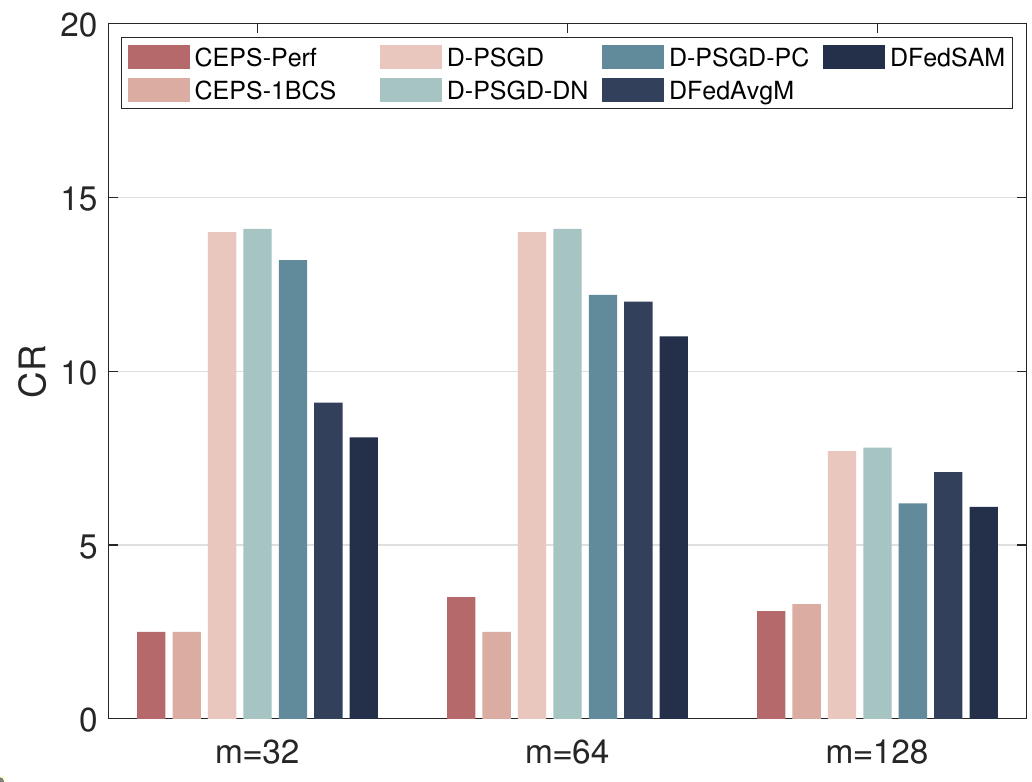}
			\includegraphics[width=0.325\textwidth]{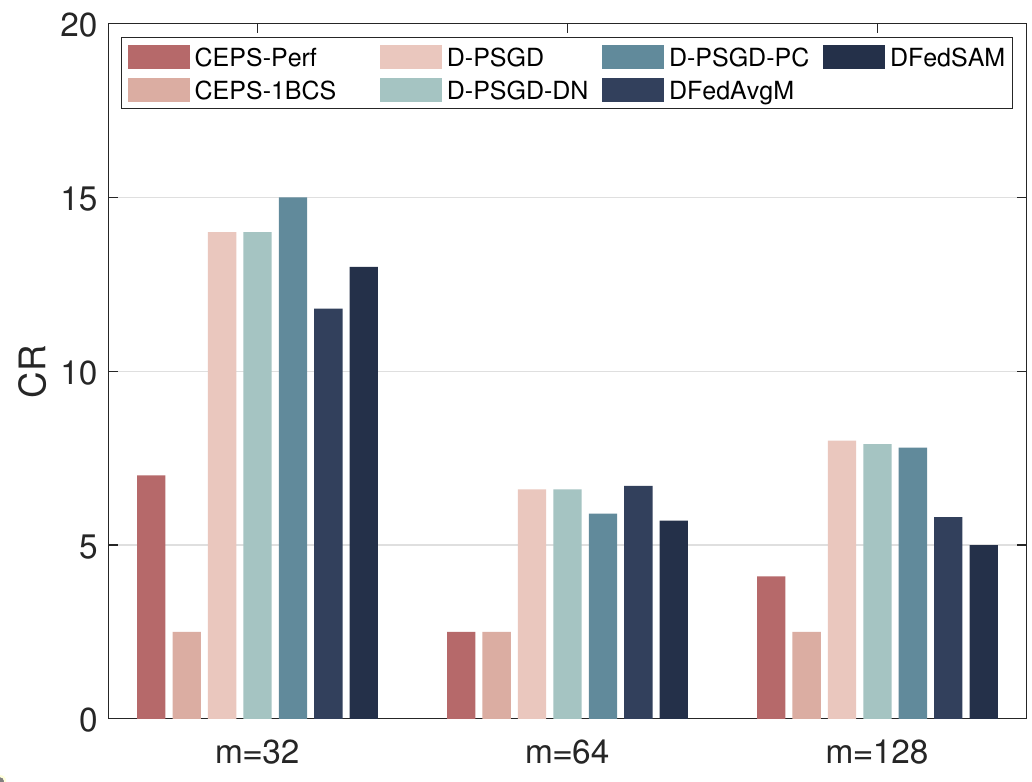} 
		\caption{CR v.s. $m$.}
		\label{example1_cr}
	\end{figure*}	

	\begin{figure*}[!t]
		\centering 
			\includegraphics[width=0.325\textwidth]{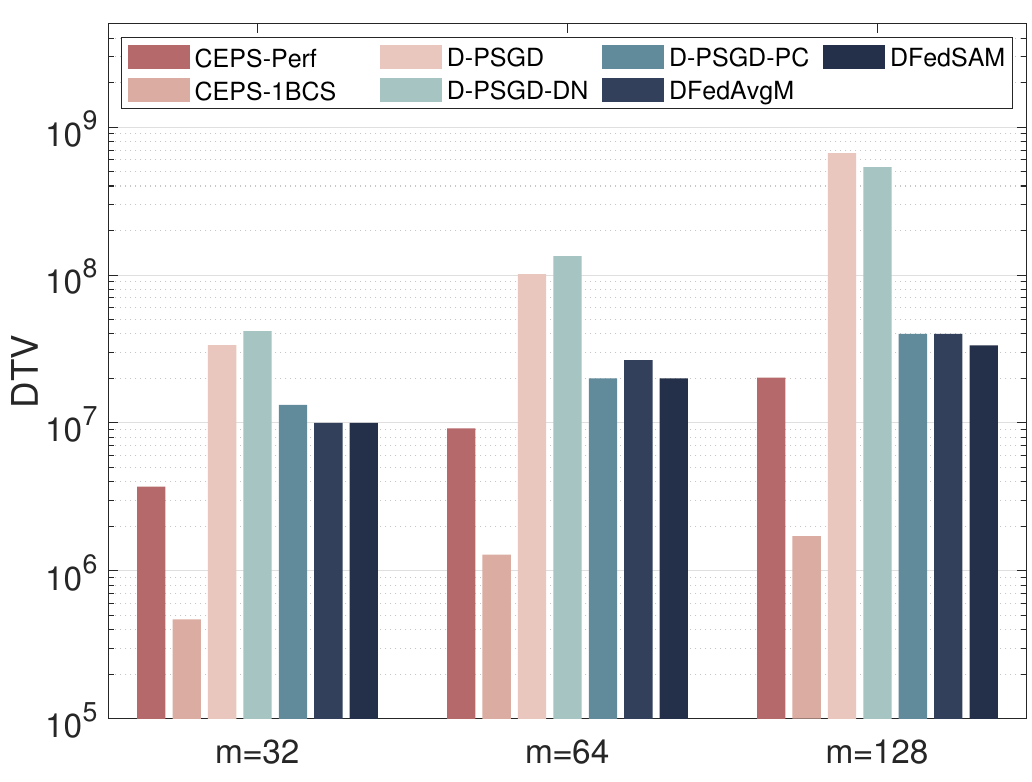}
			\includegraphics[width=0.325\textwidth]{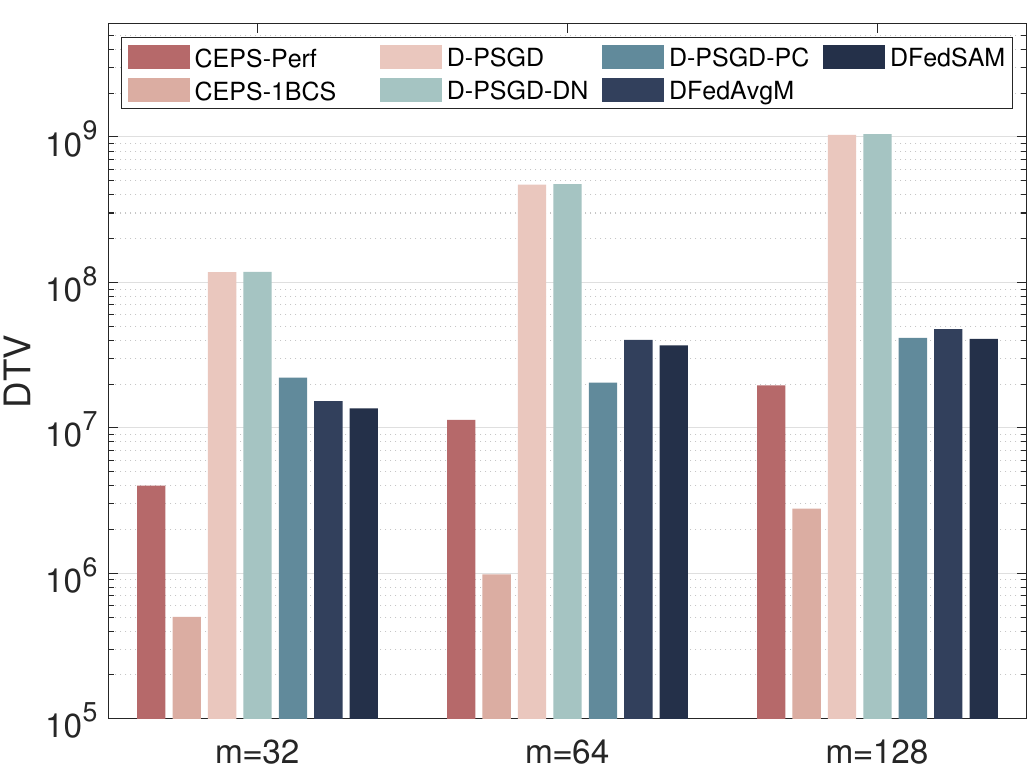}
			\includegraphics[width=0.325\textwidth]{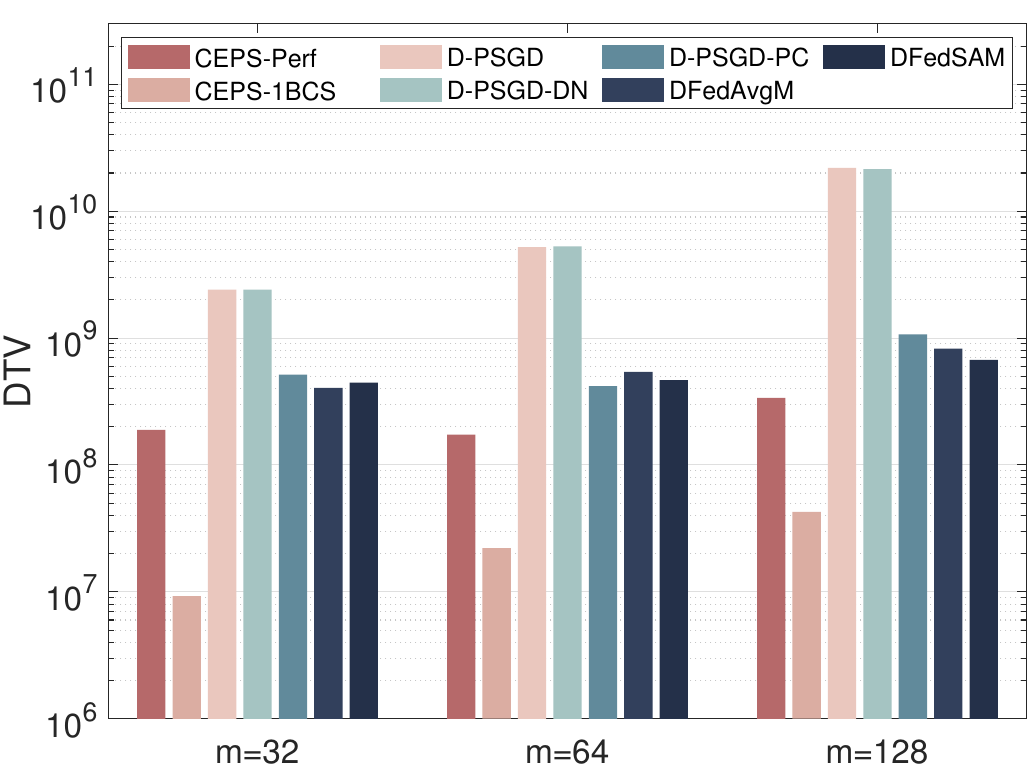} 
		\caption{Data transmission volume (DTV) v.s. $m$.}
		\label{example1_data}
	\end{figure*}	
	
		\begin{figure*}[!t]
		\centering 
			\includegraphics[width=0.325\textwidth]{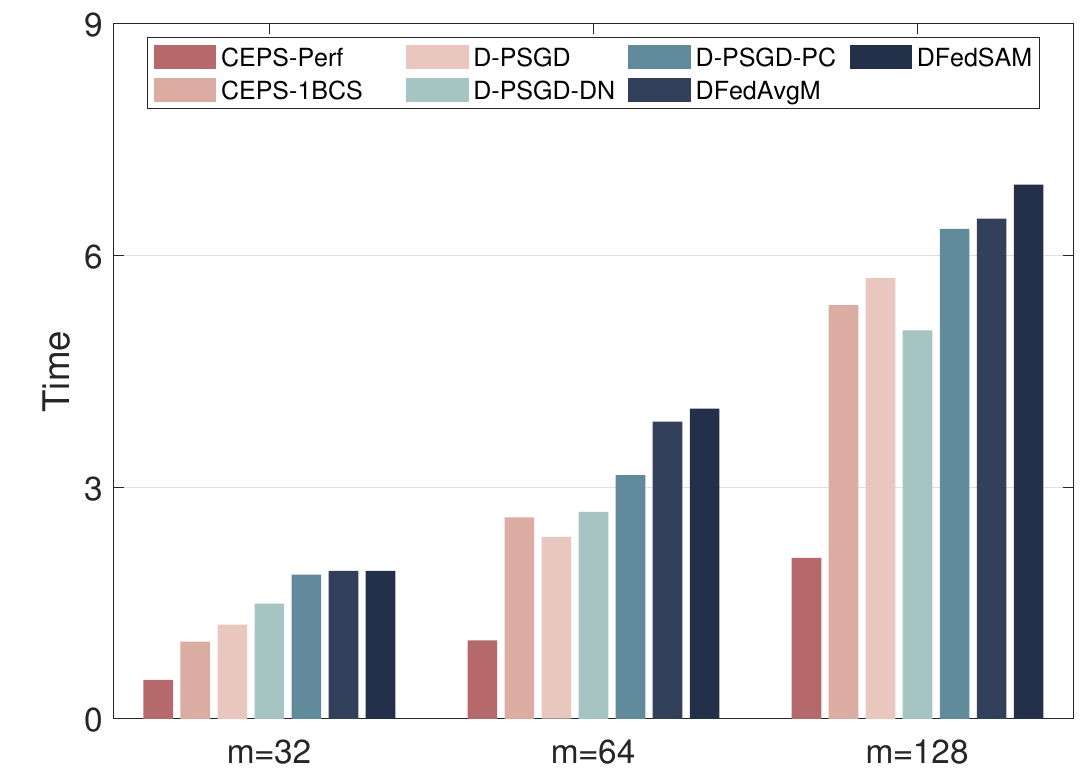}
			\includegraphics[width=0.325\textwidth]{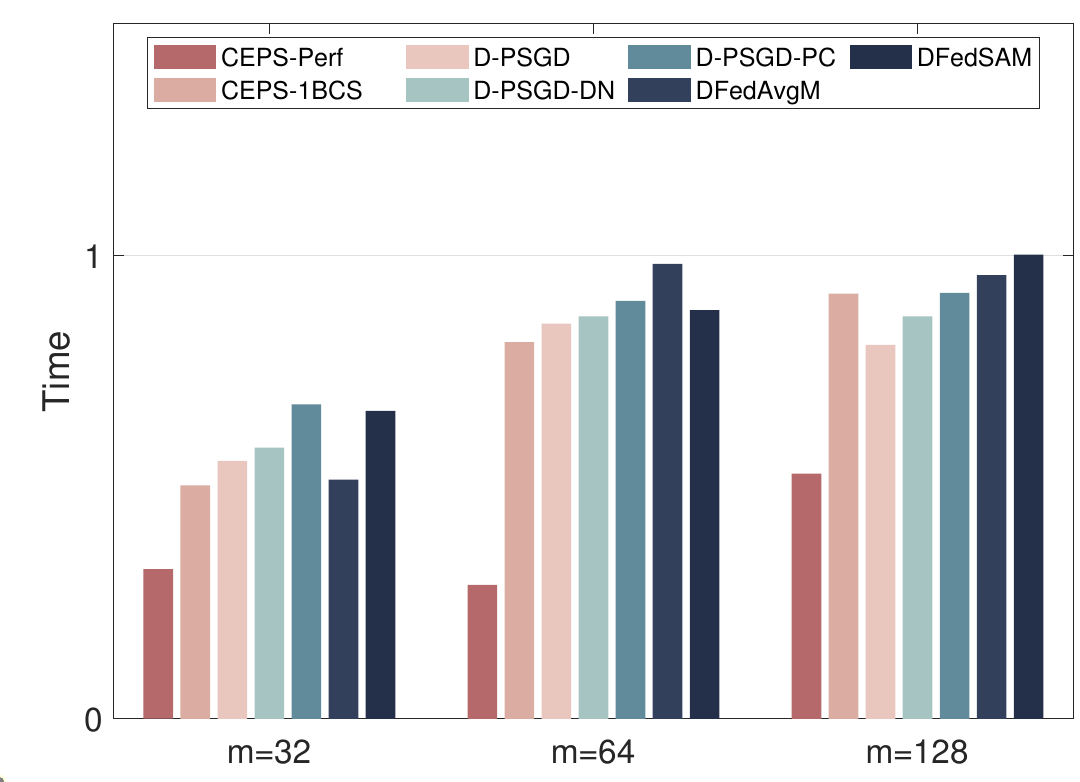}
			\includegraphics[width=0.325\textwidth]{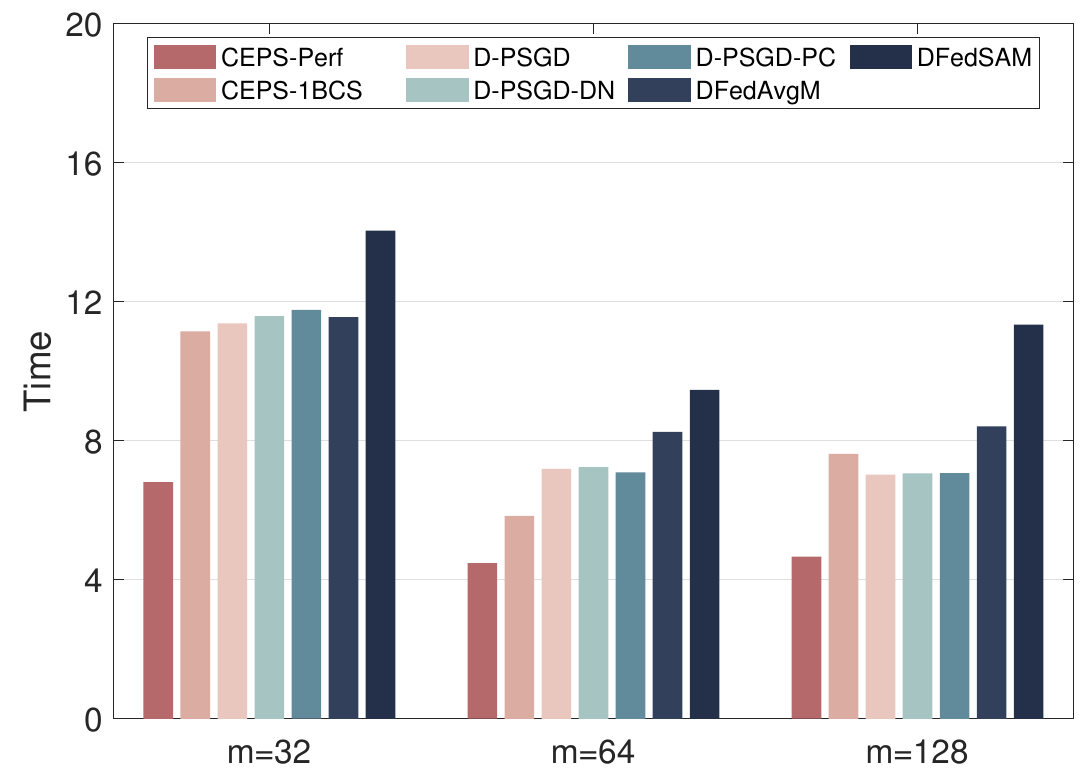} 
		\caption{Time v.s. $m$.}
		\label{example1_time}
	\end{figure*}	
	
	\section{Conclusion}
	\label{Conclusion}
	
	This paper introduces a novel algorithm, CEPS, for SDFL. By addressing critical challenges such as communication overhead, computational efficiency, robustness against stragglers, and privacy protection, CEPS provides a promising solution to improve the performance of DFL systems. 
    By integrating the 1BCS technique, the algorithm significantly reduces communication rounds and data transmission,  setting it apart from existing approaches. Moreover, CEPS leverages closed-form solutions for computational efficiency and incorporates the Gaussian mechanism for rigorous privacy protection. The theoretical convergence guarantee and numerical experiments further validate its effectiveness. Overall, CEPS stands as a promising approach for SDFL, particularly in challenging communication environments.
	
	\bibliographystyle{IEEEtran}
	\bibliography{CEPS}
	
	\newpage
	\renewcommand\thesubsection{\Alph{subsection}}
	\section*{Appendix}

	\subsection{Basic Inequalities}\label{appendixA}
		Under gradient Lipschitz assumption, for any $\w, \v$ we have 
		\begin{equation} \label{gralip}
		 \arraycolsep=1.5pt\def \arraystretch{1.25}
			\begin{array}{lcl}
			f_i\left(\w\right) \leq f_i\left(\v\right)+\left\langle\nabla f_i\left(\v\right), \w-\v\right\rangle + \frac{
				\ell}{2}\left\|\w-\v\right\|^2, 
		\end{array}\end{equation}
		and the same rule applies for $f$ with $m\ell$. For any $\w_i$ and $t>0$ 
		\begin{equation}\label{basiceq} 
		\arraycolsep=1.5pt\def\arraystretch{1.5}
			\begin{array}{l}
				\|\sum_{i=1}^{m} \w_i\|  \leq \sum_{i=1}^{m}\|\w_i\|,~\|\sum_{i=1}^m \w_i\|^2  \leq m \sum_{i=1}^m\|\w_i\|^2,\\
				2\left\langle\w_1, \w_2\right\rangle  \leq t\left\|\w_1\right\|^2+ \frac{1}{t}\left\|\w_2\right\|^2,\\ 
				\left\|\w_1+\w_2\right\|^2  \leq(1+t)\left\|\w_1\right\|^2+(1+1 / t)\left\|\w_2\right\|^2.
			\end{array}
		\end{equation}
To analyze the algorithm  globally, we define matrix $\mathbf{A}^{k}\in \mathbb{R}^{m \times m}$ to represent the communication topology by	
	\begin{equation}
		 \arraycolsep=1.5pt\def \arraystretch{1.25}
			\begin{array}{lcl}
			&	\mathbf{A}^{k}_{ji}:= \begin{cases} \frac{1}{t_i} & \text { if } j \in {N}_i^{k}, \\ 0 & otherwise,\end{cases} 
		\end{array} 
	\end{equation}
	and further define matrices $\mathbf{W}^{k},\overline{\mathbf{W}}^{k}, \mathbf{\Psi}^k,  \mathbf{\Xi}^k   \in \mathbb{R}^{n \times m}$ to represent the collection of local models, mixed models, gradients, and noise terms at step $k$:  
	\begin{equation}\label{def-WWPZ}
		 \arraycolsep=1.5pt\def \arraystretch{1.5}
			\begin{array}{lcl}
			\mathbf{W}^{k} &:=&\left(\mathbf{w}_1^{k},\mathbf{w}_2^{k}, \cdots, \mathbf{w}_m^{k}\right),\\	 
			\overline{\mathbf{W}}^{k}&:=&\left(\overline{\mathbf{w}}_1^{k}, \overline{\mathbf{w}}_2^{k},\cdots, \overline{\mathbf{w}}_m^{k}\right) = \mathbf{W}^{k}\mathbf{A}^{k}, \\
			\mathbf{\Psi}^k &:=&
			\left( \nabla f_1(\overline{\mathbf{w}}_1^{k}),  \nabla f_2(\overline{\mathbf{w}}_2^{k}), \cdots   \nabla f_m(\overline{\mathbf{w}}_m^{k}) \right), \\ 
			\mathbf{\Xi}^k &:=& (\boldsymbol{\xi}_1^k,\boldsymbol{\xi}_2^k,\cdots, \boldsymbol{\xi}_m^k) .
		\end{array}
	\end{equation}
    By the definition of average point
    \begin{equation}\label{avgpoint}
			 \arraycolsep=1.5pt\def \arraystretch{1.25}
			\begin{array}{lcl}
				\boldsymbol{\varpi}^{a k_0}  :=\frac{1}{m}\sum_{i=1}^m\mathbf{w}_i^{a k_0},
			\end{array}
	\end{equation}
we have 
		\begin{equation}\label{avgpoint-1}
			 \arraycolsep=1.5pt\def \arraystretch{1.25}
			\begin{array}{lcl}
				\boldsymbol{\varpi}^{\alpha}  =\frac{1}{m}\sum_{i=1}^m\mathbf{w}_i^{\alpha k_0}=\frac{\mathbf{W}^{\alpha k_0} \mathbf{1}_m}{m},~\text{where} ~\alpha &:= ak_0.
			\end{array}
	\end{equation} 
	For simplicity, for each $a=1,2,3,\ldots$, define
\begin{equation}\label{def-YPe}
			 \arraycolsep=1.5pt\def \arraystretch{1.25}
			\begin{array}{lcl}
				\U^{a k_0} &:= \left(\mathbf{u}_1^{a k_0}, \cdots, \mathbf{u}_m^{a k_0}\right)  {=} c \W^{a k_0}\A^{a k_0} - \Ps^{a k_0},
			\end{array}
		\end{equation}
        where the last equality holds by the updating step of Algorithm 1: $\mathbf{u}_i^{k}= \sigma_i m_i^{k} \overline{\mathbf{w}}_i^k-\nabla f_i(\overline{\mathbf{w}}_i^k)$.
		And let		 
				\begin{equation}\label{def-YPe-alpha}
			 \arraycolsep=1.5pt\def \arraystretch{1.25}
			\begin{array}{lcl}
				 \A^{i\to a} &:=& 
				\left\{
				\begin{array}{ll}
				\Pi_{t=(i+1) k_0}^{(a-1) k_0}\A^t,&i=0,\ldots, a-2,\\
				\textbf{I},~~&i=a-1,
				\end{array}
				\right.\\
				 \mathbf{d}_j &:=& \frac{\om}{m}-\mathbf{e}_j, 
			\end{array}
		\end{equation} 
	 		where $\mathbf{e}_j$ is the $j$th column of identity matrix.  One can verify that  $\boldsymbol{\phi}\om^\top \mathbf{d}_j =\boldsymbol{\phi}-\boldsymbol{\phi}=0$. Finally, recall we have defined
        $c_0 := \left(\frac{ 80m\sqrt{\ell} }{1-\tau}\right)^2$ and $ \tau :=\left(1-\frac{1}{m^{mB}}\right)^{\frac{1}{B}}$, then
				\begin{equation*}
				 \arraycolsep=1.5pt\def \arraystretch{1.25}
			\begin{array}{l}
			c  \geq c_0=\left(\frac{ 80m\sqrt{\ell} }{1-\tau}\right)^2= \frac{ 6400m^2 \ell  }{(1-\tau)^2}  
			 \geq   \frac{ 80m \ell }{1-\tau}, 
			\end{array}
		\end{equation*} 
		which gives rise to
		\begin{equation}\label{lower-bd-c-tau}
		 \arraycolsep=1.5pt\def \arraystretch{1.75}
			\begin{array}{l}
			c^2(1-{\tau})^2-384 m^2\ell^2 ~\geq~  \frac{c^2(1-{\tau})^2}{2},~ c ~\geq~  2\ell,\\
			 \frac{1}{8cm}  ~\geq~ \frac{800m\ell}{c^2(1-\tau)^2},~\frac{c}{2} ~\geq~ \frac{800m^2\ell}{3(1-\tau)^2}.
			\end{array}
		\end{equation}

	\subsection{Key Lemma}\label{appendixB}
	\begin{lemma}\label{lemmaproduct}
		The product of  matrix $\A^{i\to\alpha}$ is bounded by 
		\begin{equation}\label{productbound}
			 \arraycolsep=1.5pt\def \arraystretch{1.25}
			\begin{array}{lcl}
				\left\|\A^{i\to {a}}\mathbf{d}_j\right\|^2
				\leq 64m \tau^{2(a-i-1)}, ~~ \forall i \leq {a}-1.
			\end{array}
		\end{equation}
	\end{lemma}
	\begin{proof}  When $i=a-1$, $\A^{i\to {a}}=\I$ which  suffices to
	\begin{equation}\label{productbound-a-1}
			 \arraycolsep=1.5pt\def \arraystretch{1.25}
			\begin{array}{lcl}
				\left\|\A^{i\to {a}}\mathbf{d}_j\right\|^2=\left\|\mathbf{d}_j\right\|^2 = \frac{m-1}{m}\leq 64m\tau^{2(\alpha-i-1)}.
			\end{array}
		\end{equation}
When $i< a-1$, by Assumption 4,  graph ${G}_t^B$ is connected and thus is strongly connected as it is  undirected. Then  by   \cite[Corollary 2]{nedic2014distributed}, there is a stochastic vector $\boldsymbol{\phi}$ such that 
		\begin{equation}
			\left|[\A^{i\to a}-\boldsymbol{\phi}\om^\top]_{rt}\right|=\left|[\A^{i\to a}]_{rt}-[\boldsymbol{\phi}]_{r}\right| \leq 4 \tau^{a-i-2},
			\label{columnstochasticinequal}
		\end{equation} 
for $\forall r,t\in V$. The above condition leads to
		\begin{equation*} \arraycolsep=1.5pt\def\arraystretch{1.5}
			\begin{array}{lcl}			
			\left|\left[(\A^{i\to a}-\boldsymbol{\phi}\om^\top)\mathbf{d}_j\right]_r\right|
			&\leq& \max_t\left|\left[\A^{i\to a}-\boldsymbol{\phi}\om^\top\right]_{rt}\right|\left\|\mathbf{d}_j\right\|_1\\
			&\leq &	
			 \frac{8(m-1)\tau^{a-i-2}}{m}, ~~\forall r\in V,
			\end{array}
		\end{equation*}
	where $\|\w\|_1$ stands for the 1-norm, and thus
		\begin{equation*}\arraycolsep=1.5pt\def\arraystretch{1.5}
			\begin{array}{lcl}
			\left\|\A^{i\to a}\mathbf{d}_j\right\|^2=\left\|(\A^{i\to a}-\boldsymbol{\phi}\om^\top)\mathbf{d}_j\right\|^2
			&\leq \frac{64(m-1)^2\tau^{2(a-i-2)}}{m}\\
			&\leq {64m\tau^{2(\alpha-i-1)}},
		\end{array}
		\end{equation*} 
		where the first equality is from $\boldsymbol{\phi}\om^\top\mathbf{d}_j=0$, the last inequality is from $m\geq2$ and
		$$\begin{array}{l}
		\tau^2 = \left(1-\frac{1}{m^{mB}}\right)^{\frac{2}{B}} \geq \left( 1-\frac{1}{m^{mB}} \right)^2\geq \left(1-\frac{1}{m^{m}} \right)^2\geq \frac{m-1}{m}. \end{array}
		$$
Together with \eqref{productbound-a-1}, we show the desired result.
	\end{proof}
\begin{lemma}\label{lemma-e-infty} Consider  $k_0$ consecutive iterations: $(a-1)k_0, (a-1)k_0+1, \ldots, ak_0-1$ for some integer $a>0$. For $t=1,\ldots,k_0-1,$ define 
		\begin{equation} \label{def-Q-P}
		\arraycolsep=1.5pt\def \arraystretch{1.75}
		\begin{array}{lcl}		
		\Q^{ak_0-k_0}&:=& \PO\left( \frac{\U^{({a}-1) k_0}+\mathbf{\Xi}^{({a}-1) k_0}}{c} \right) - \frac{\U^{({a}-1) k_0}}{c},\\ 
		 \Q^{ak_0-t}&:=&\PO\left(\frac{\U^{({a}-1) k_0}+\mu\W^{ak_0-t}}{c+\mu}\right)-\frac{\U^{({a}-1) k_0}+\mu\W^{ak_0-t}}{c+\mu},\\ 
		 \P^{({a}-1) k_0}&:=&\sum_{i=0}^{k_0-1}  \rho^{i} \Q^{ak_0-1-i},
		  ~\rho :=\frac{\mu}{c+\mu},  
		\end{array}
		\end{equation}
		and 	an error bound
		\begin{equation} \label{def-error-project}
				 \arraycolsep=1.5pt\def \arraystretch{1.25}
				 \begin{array}{ll}				
				e_{\infty}:= \underset{k\geq 1}{\sup}&\left\{\left\|\PO\left( {\U^{k}}+\mathbf{\Xi}^{k}\right) -{\U^{k}}\right\|^2, \right.\\
				&\left.~\left\|\PO\left({\U^{k}+\mu\W^{k}}\right)-\left({\U^{k}+\mu\W^{k}}\right)\right\|^2\right\}.
	\end{array}
		\end{equation}
	Then for any $a=1,2,3,\ldots,$ it holds that
					\begin{equation} \label{bd-E-P}
			 \arraycolsep=1.5pt\def \arraystretch{1.25}
			\begin{array}{lcl}
				\mathbb{E} \left\| \P^{({a}-1) k_0} \right\|^2 \leq   
				  \frac{k_0  e_{\infty}}{c^2}.
			\end{array}
		\end{equation}
		\end{lemma}
	\begin{proof} Denote $\U:= \U^{({a}-1) k_0}, \mathbf{\Xi}=\mathbf{\Xi}^{({a}-1) k_0}$, and $\W:=\W^{ak_0-t}$. It follows
		\begin{equation*}
			\begin{array}{ll}
				 \left\|\Q^{ak_0-k_0}\right\|^2= \left\|\PO\left( \frac{\U+\mathbf{\Xi}}{c} \right)- \frac{\U}{c} \right\|^2 =  \frac{\left\|\PO\left(  {\U+\mathbf{\Xi}}  \right)-  {\U}  \right\|^2}{c^2} \overset{\eqref{def-error-project}}{\leq} \frac{e_{\infty}}{c^2}.
			\end{array}
		\end{equation*}
	As $\U^{({a}-1) k_0}=\U^{ {a} k_0-t}$ for any $t=1,\ldots,k_0-1$, we have
			\begin{equation*}\arraycolsep=1.5pt\def\arraystretch{1.75}
			\begin{array}{lcl}
			 \left\|\Q^{ak_0-t}\right\|^2 &=& \left\|\PO\left( \frac{\U+\mu\W}{c+\mu} \right)- \frac{\U+\mu\W}{c+\mu}  \right\|^2 \\&= & \frac{\left\|\PO\left(  \U+\mu\W   \right)- \left(  \U+\mu\W   \right) \right\|^2}{(c+\mu)^2} ~\overset{\eqref{def-error-project}}{\leq} ~\frac{e_{\infty}}{(c+\mu)^2}.
			\end{array}
		\end{equation*}
		Using the above two conditions, we obtain	 
			\begin{equation*} 
			 \arraycolsep=1.5pt\def \arraystretch{1.75}
			\begin{array}{lcl}
				\mathbb{E} \left\| \P^{({a}-1) k_0} \right\|^2 
				&=& \mathbb{E}\left\|  \sum_{i=0}^{k_0-1} \rho^{i} \Q^{ak_0-1-i}\right\|^2 \\ 
				 &\leq &k_0\sum_{i=0}^{k_0-1}\rho^{2i}\mathbb{E}\left\|    \Q^{ak_0-1-i}\right\|^2\\
				 &\leq& k_0\sum_{i=0}^{k_0-2} \frac{\rho^{2i} e_{\infty}}{(c+\mu)^2} + \frac{k_0 \rho^{2(k_0-1)} e_{\infty}}{c^2}  \\
				&=& \frac{k_0 (1-\rho^{2(k_0-1)}) e_{\infty}}{(c+\mu)^2-\mu^2}  + \frac{k_0 \rho^{2(k_0-1)} e_{\infty}}{c^2} ~\leq~    \frac{k_0  e_{\infty} }{c^2}.
			\end{array}
		\end{equation*}
	The proof is completed. 
\end{proof}	 	
	Hereafter, for notational convenience,  let
		$\alpha:=a k_0$. 
	\begin{lemma} Under Assumption 1-2 and setting  ${c\geq c_0}$, it holds, 	 
		\begin{equation}	\label{lemmagradient}
			 \arraycolsep=1.5pt\def \arraystretch{1.75}
			\begin{array}{lcl}
				&&\frac{1}{2 c m} \sum_{a=0}^{T-1}\mathbb{E}\left\|\nabla f\left(\boldsymbol{\varpi}^{\alpha}\right)-\Ps^{\alpha} \om\right\|^2\\
				 
				&+&\frac{m\ell}{2}\sum_{a=0}^{T-1}\mathbb{E}\left\|\frac{\W^{ak_0}(\A^{\alpha}-\mathbb{E}{\A}^{\alpha})\om}{m}\right\|^2 \\
				&\leq&\frac{800m\ell}{c^2(1-\tau)^2}\left( {m^2\zeta^2T}  +    \sum_{a=0}^{T-1}\mathbb{E}\left\|\nabla f\left(\boldsymbol{\varpi}^{\alpha}\right)\right\|^2\right. \\ 
						&&\hspace{25mm} +  \left. \frac{mc^2}{3} \sum_{a=0}^{T-1}  \mathbb{E}\left\|\P^{\alpha}\right\|^2\right).
			\end{array}
		\end{equation}	 
	\end{lemma}
	\begin{proof}We begin with rewriting the update rule of Algorithm 1 when $\kappa_i=k_0, \forall i$. For $t=1,2,\ldots,k_0-1$		 
		\begin{equation} \label{facts-w-a-t}
			 \arraycolsep=1.5pt\def \arraystretch{1.25}
			\begin{array}{lcl}		
\W^{({a}-1) k_0+1}&{=} &~ \PO\left( \frac{\U^{({a}-1) k_0}+\mathbf{\Xi}^{({a}-1) k_0}}{c} \right)\\
&\overset{\eqref{def-Q-P}}{=} &~\Q^{({a}-1) k_0}+ \frac{\U^{({a}-1) k_0}}{c},\\
\W^{ak_0-t+1}&{=}&~\PO\left(\frac{\U^{({a}-1) k_0}+\mu\W^{ak_0-t}}{c+\mu}\right)\\
&\overset{\eqref{def-Q-P}}{=} &~ \Q^{ak_0-t} +\frac{\U^{({a}-1) k_0}+\mu\W^{ak_0-t
}}{c+\mu}.
			\end{array}
		\end{equation}	 
 Using the above condition, we have
	\begin{equation*}
				 \arraycolsep=1.5pt\def\arraystretch{1.5}
			\begin{array}{lcl}
				&&\W^{ak_0}	\overset{\eqref{facts-w-a-t}}{=}  \Q^{ak_0-1} +\frac{\U^{({a}-1) k_0}+\mu\W^{ak_0-1}}{c+\mu}\\[1ex]
				&{=}&\Q^{ak_0-1} +\frac{\U^{({a}-1) k_0} }{c+\mu}+\rho\W^{ak_0-1}\\	
					&\overset{\eqref{facts-w-a-t}}{=} &    \Q^{ak_0-1}+ \rho \Q^{ak_0-2} 
				 +\frac{\left(1+ \rho\right)\U^{({a}-1) k_0}}{c+\mu}+ \rho^2 \W^{ak_0-2}  \\[1ex]
					&{=}&\cdots\\
					&{=}& \sum_{i=0}^{k_0-2}  \left(\rho^{i} \Q^{ak_0-i-1}   +   \frac{\rho^{i}\U^{({a}-1) k_0}}{c+\mu}\right) +\rho^{ k_0-1}  \W^{({a}-1) k_0+1 }\\
						&\overset{\eqref{facts-w-a-t}}{=} & \sum_{i=0}^{k_0-2}  \rho^{i} \Q^{ak_0-i-1}    +  \frac{\left(1-\rho^{ k_0-1} \right)\U^{({a}-1) k_0}}{c}+  \frac{\rho^{ k_0-1} \U^{({a}-1) k_0 }}{c}\\					
						&\overset{\eqref{def-Q-P}}{=} & \P^{({a}-1) k_0}   +   \frac{ \U^{({a}-1) k_0}}{c},  							\end{array}
		\end{equation*} 	  
		which further results in 
		\begin{equation}
				 \arraycolsep=1.5pt\def \arraystretch{1.5}
			\begin{array}{lcl}
					&&\W^{{a} k_0}   {=}   	 \frac{ \U^{({a}-1) k_0}}{c}  + \mathbf{P}^{({a}-1)k_0} \\[1ex]
				 &\overset{\eqref{def-YPe}}{=} &  \W^{({a}-1) k_0}\A^{({a}-1)k_0} -\frac{1}{c}\Ps^{({a}-1)k_0}  + \mathbf{P}^{({a}-1)k_0} \\
				&{=} & \W^{({a}-2) k_0}\A^{({a}-2)k_0}\A^{({a}-1)k_0}\\
				&-&\frac{1}{c}\left(\Ps^{({a}-1)k_0}+\Ps^{({a}-2)k_0}\A^{({a}-1)k_0}\right)\\
				&+&\left(\P^{({a}-1)k_0}+\P^{({a}-2)k_0}\A^{({a}-1)k_0}\right)
				\\
				&{=} & \cdots \\
				&{=} &  \W^0\prod_{i=0}^{{a}-1}\A^{ik_0}   -\frac{1}{c} \sum_{i=0}^{{a}-1}\Ps^{ik_0}\prod_{q=i+1}^{{a}-1} \A^{q k_0} \\
				&+& \sum_{i=0}^{{a}-1}\mathbf{P}^{ik_0}\prod_{q=i+1}^{{a}-1} \A^{q k_0} ,
			\end{array}
			\label{commustep}
		\end{equation} 
		where $\prod_{q=i+1}^{{a}-1} \A^{qk_0} = \mathbf{I}$ when $i={a}-1$. With $\W^0 =0$, we perform two consecutive communication steps and obtain
		\begin{equation}\label{commuiter}
			 \arraycolsep=1.5pt\def \arraystretch{1.25}
			\begin{array}{lcl}
				\W^{\alpha }  \overset{(\ref{commustep},\ref{def-YPe-alpha})}{=}  -\frac{1}{c} \sum_{i=0}^{a-1}\Ps^{ik_0}\A^{i\to a}
				+ \sum_{i=0}^{a-1}\mathbf{P}^{ik_0}\A^{i\to a}.
			\end{array}
		\end{equation}
		We divide the proof into four parts:
		\vspace*{-1mm}
		\begin{proofpart}
			\label{part1}
			Direct verification enables us to derive the following chain of inequalities,		 
			\begin{equation}\label{parta-fact-1}
				 \arraycolsep=1.5pt\def \arraystretch{1.5}
			\begin{array}{lcl}
					&&\mathbb{E}\left\|\nabla f\left(\boldsymbol{\varpi}^{\alpha}\right)-\Ps^{\alpha } \om\right\|^2 \\
					&{=}&\mathbb{E}\left\|\sum_{i=1}^{m}\nabla f_i\left(\boldsymbol{\varpi}^{\alpha}\right) - \sum_{i=1}^{m}\nabla f_i\left(\overline{\w}_i^\alpha\right)\right\|^2 \\ 
					 &\overset{(\ref{basiceq})}{\leq}&
					m\sum_{i=1}^{m}\mathbb{E}\left\|\nabla f_i\left(\boldsymbol{\varpi}^{\alpha}\right) - \nabla f_i\left(\overline{\w}_i^\alpha\right)\right\|^2\\  
					&\overset{(\ref{gralip})}{\leq}&
					m \ell^2 \sum_{i=1}^m \mathbb{E}\left\|\boldsymbol{\varpi}^{\alpha}-\overline{\w}_i^{\alpha }\right\|^2\\
				&{=}& m \ell^2 \sum_{i=1}^{m}\mathbb{E}\|
					\frac{1}{t_i}\sum_{j \in {N}_i^{\alpha}}\left(\boldsymbol{\varpi}^{\alpha} - \w_j^\alpha\right)\|^2\\ 
					&\overset{(\ref{basiceq})}{\leq}& m \ell^2 \sum_{i=1}^{m}\frac{1}{t_i}\sum_{j \in {N}_i^{\alpha}}\mathbb{E}\left\|\boldsymbol{\varpi}^{\alpha} - \w_j^\alpha\right\|^2,
				\end{array}
			\end{equation} 
			where the second equality holds also due to $|{N}_i^\alpha|=m_i^\alpha = t_i $ for any $i$ and $\alpha.$
		\end{proofpart}
		\vspace*{-3mm}
		\begin{proofpart}
			\label{part3}
			By the global iteration rule, we have
			\begin{equation}
				 \arraycolsep=1.5pt\def \arraystretch{1.25}
			\begin{array}{lcl}
					&&\mathbb{E}\left\|\boldsymbol{\varpi}^{\alpha}-\w_j^{\alpha }\right\|^2
					\overset{(\ref{avgpoint})}{=}\mathbb{E}\left\|\frac{\W^{\alpha}\om}{m}-\w_j^{\alpha }\right\|^2\\
					&\overset{(\ref{def-WWPZ})}{=}& \mathbb{E}\left\|\frac{\W^{\alpha}\om}{m}-\W^{\alpha} \mathbf{e}_j\right\|^2 \\
					&\overset{(\ref{commuiter})}{=}&\mathbb{E}\left\|\left(	-\frac{1}{c} \sum_{i=0}^{a-1}\Ps^{ik_0}\A^{i\to a}
					+ \sum_{i=0}^{a-1}\mathbf{P}^{ik_0}\A^{i\to a}\right)
					\left(\frac{\om}{m}-\mathbf{e}_j\right)\right\|^2 \\
					&\overset{(\ref{basiceq})}{\leq}&\frac{2}{c^2}\mathbb{E}\left\|\sum_{i=0}^{a-1}\Ps^{ik_0}\A^{i\to a}\mathbf{d}_j\right\|^2
					+ 2\mathbb{E}\left\|\sum_{i=0}^{a-1} \P^{ik_0}\A^{i\to a}\mathbf{d}_j\right\|^2.
				\end{array}
				\label{threeterms}
			\end{equation}
		\end{proofpart}
		\begin{proofpart}
			\label{part4}
			Note that the terms in (\ref{threeterms}) share a similar structure, thus we give a detailed calculation for the first term, and the other one follows the same strategy. First, consider bounding the gradient matrix as follows:			 
			\begin{equation}\label{gradbound}
				 \arraycolsep=1.5pt\def \arraystretch{1.25}
			\begin{array}{lcl}
					&&\mathbb{E}\left\|\Ps^{ik_0}\right\|^2\\
					 &\overset{(\ref{basiceq})}{\leq} &3\mathbb{E}\left\|\Ps^{ik_0}-\mathbf{B}^{ik_0}\right\|^2 + 3\mathbb{E}\left\|\nabla f\left(\boldsymbol{\varpi}^{ik_0}\right)\frac{\om^\top}{m}\right\|^2  \\
					 &+& 3\mathbb{E}\left\|\mathbf{B}^{ik_0}-\nabla f\left(\boldsymbol{\varpi}^{ik_0}\right)\frac{\om^\top}{m}\right\|^2 \\
					 &\overset{(\ref{gralip})}{\leq}& 3\ell^2\sum_{p=1}^{m}\mathbb{E}\left\|\boldsymbol{\varpi}^{ik_0} - \overline{\w}_p^{ik_0}\right\|^2 + 3m\zeta^2 + \frac{3}{m}\mathbb{E}\left\|\nabla f\left(\boldsymbol{\varpi}^{ik_0}\right)\right\|^2,
				\end{array}
			\end{equation}	 
			where ${\mathbf{B}^{ik_0}:= (\nabla f_1(\boldsymbol{\varpi}^{ik_0}),\cdots, \nabla f_m(\boldsymbol{\varpi}^{ik_0}))}.$ Moreover, direct calculation derives that			 
			\begin{equation*}
				 \arraycolsep=1.5pt\def \arraystretch{1.75}
			\begin{array}{lcl}
					 &&\mathbb{E}\left\|\sum_{i=0}^{a-1}\Ps^{ik_0}\A^{i\to{a}}\mathbf{d}_j\right\|^2\\ 
					&{=}& \sum_{i=0}^{a-1}\mathbb{E}\left\|\Ps^{ik_0}\A^{i\to{a}}\mathbf{d}_j\right\|^2 \\
					&+&2\sum_{i=0}^{a-1}\sum_{{t}= i+1}^{a-1}\mathbb{E}\left\langle\Ps^{ik_0}\A^{i\to{a}}\mathbf{d}_j,\Ps^{{tk_0}}\A^{{t}\to a}\mathbf{d}_j\right\rangle \\
					&{\leq}&\sum_{i=0}^{a-1}\mathbb{E}\left\|\Ps^{ik_0}\right\|^2\left\|\A^{i\to{a}}\mathbf{d}_j\right\|^2 \\
					&+& 2\sum_{i=0}^{a-1}\sum_{{t}= i+1}^{a-1}\mathbb{E}\left\|\Ps^{ik_0}\right\|\left\|\Ps^{{tk_0}}\right\| \left\|\A^{i\to{a}}\mathbf{d}_j\right\| \left\|\A^{{t}\to a}\mathbf{d}_j\right\|\\
					&{\leq}& \sum_{i=0}^{a-1}\mathbb{E}\left\|\Ps^{ik_0}\right\|^2\left\|\A^{i\to{a}}\mathbf{d}_j\right\|^2 +\sum_{i=0}^{a-1}\sum_{{t}= i+1}^{a-1}\Big(\mathbb{E}\left\|\Ps^{ik_0}\right\|^2\\
					&+&\mathbb{E}\left\|\Ps^{{tk_0}}\right\|^2\Big) \left\|\A^{i\to{a}}\mathbf{d}_j\right\| \left\|\A^{{t}\to a}\mathbf{d}_j\right\|
					\\
					&{=}&
				 \sum_{i=0}^{a-1}\mathbb{E}\left\|\Ps^{ik_0}\right\|^2 \left\|\A^{i\to{a}}\mathbf{d}_j\right\| \sum_{{t}= 0}^{a-1}  \left\|\A^{{t}\to a}\mathbf{d}_j\right\|, 
				\end{array}
			\end{equation*}			 
		which results in
		\begin{equation}
		 \arraycolsep=1.5pt\def \arraystretch{1.75}
			\begin{array}{lcl}
			&&\mathbb{E}\left\|\sum_{i=0}^{a-1}\Ps^{ik_0}\A^{i\to{a}}\mathbf{d}_j\right\|^2\\ 
			&\overset{(\ref{productbound})}{\leq}&\sum_{i=0}^{a-1}\mathbb{E}\left\|\Ps^{ik_0}\right\|^2 64m\tau^{\alpha-i-1}\sum_{{t}=0}^{a-1}\tau^{a-t-1} \\
			&\leq& \sum_{i=0}^{a-1}\mathbb{E}\left\|\Ps^{ik_0}\right\|^2\frac{64m{\tau}^{a-i-1}}{1-{\tau}}.
		\end{array}
		\end{equation}
			Similar reasoning to show the above condition enables us to prove that			
			\begin{equation*}\label{firstterm-1}
				 \arraycolsep=1.5pt\def \arraystretch{1.25}
			\begin{array}{lcl}
					\mathbb{E}\left\|\sum_{i=0}^{a-1} \P^{ik_0}\A^{i\to{a}}\mathbf{d}_j\right\|^2 
					\leq \sum_{i=0}^{a-1}\mathbb{E}\left\|\P^{ik_0}\right\|^2\frac{64m{\tau}^{a-i-1}}{1-{\tau}}.
				\end{array}
			\end{equation*}			
			\begin{proofpart}
				\label{part5}
				The above two conditions immediately give rise to	 
				\begin{equation*} 
					 \arraycolsep=1.5pt\def \arraystretch{2}
			\begin{array}{lcl}
						 &&\sum_{a=0}^{T-1}\sum_{q=1}^m \mathbb{E}\left\|\boldsymbol{\varpi}^{\alpha}-\overline{\w}_q^{\alpha }\right\|^2  
						\\
						& \overset{(\ref{basiceq})}{\leq} &  \sum_{a=0}^{T-1}\sum_{q=1}^{m}\frac{1}{t_q}\sum_{j \in \mathrm{N}_q^{\alpha}}\mathbb{E}\left\|\boldsymbol{\varpi}^{\alpha} - \w_j^{\alpha}\right\|^2 \\
						&\overset{(\ref{threeterms})}{\leq}&
						\sum_{a=0}^{T-1}\sum_{q=1}^{m}\frac{1}{t_q}\sum_{j \in \mathrm{N}_q^{\alpha}}\Big(\frac{2}{c^2}\mathbb{E}\left\|\sum_{i=0}^{a-1}\Ps^{ik_0}\A^{i\to a}\mathbf{d}_j\right\|^2 \\
						&+& 2\mathbb{E}\left\|\sum_{i=0}^{a-1} \P^{ik_0}\A^{i\to a}\mathbf{d}_j\right\|^2\Big)\\
						 &{\leq}&
						\sum_{a=0}^{T-1}\sum_{q=1}^{m}\frac{1}{t_q}\sum_{j \in \mathrm{N}_q^{\alpha}}\sum_{i=0}^{a-1}(\frac{1}{c^2}\mathbb{E}\left\|\Ps^{ik_0}\right\|^2\\
						&+&  \mathbb{E}\left\|\P^{ik_0}\right\|^2)\frac{128m{\tau}^{a-i-1}}{1-{\tau}}\\	
						&{=}&
						\sum_{a=0}^{T-1} m\sum_{i=0}^{a-1}\left(\frac{1}{c^2}\mathbb{E}\left\|\Ps^{ik_0}\right\|^2+  \mathbb{E}\left\|\P^{ik_0}\right\|^2\right)\frac{128m{\tau}^{a-i-1}}{1-{\tau}}\\
						&{\leq}&
						 \frac{128m^2 }{ (1-{\tau})^2} \sum_{i=0}^{T-1}   \left(\frac{1}{c^2}\mathbb{E}\left\|\Ps^{ik_0}\right\|^2+\mathbb{E}\left\|\P^{ik_0}\right\|^2\right)  \\
					&\overset{(\ref{gradbound})}{\leq}  &
						 \frac{128m^2 }{c^2(1-{\tau})^2} \sum_{i=0}^{T-1} \left(3\ell^2\sum_{p=1}^{m}\mathbb{E}\left\|\boldsymbol{\varpi}^{ik_0} - \overline{\w}_p^{ik_0}\right\|^2 + 3m\zeta^2 \right.\\
						 &+&\left. \frac{3}{m}\mathbb{E}\left\|\nabla f\left(\boldsymbol{\varpi}^{ik_0}\right)\right\|^2\right)  
						+ \frac{128m^2}{(1-{\tau})^2} \sum_{i=0}^{T-1}  \mathbb{E}\left\|\P^{ik_0}\right\|^2,
					\end{array}
				\end{equation*}			 
				which further leads to
				\begin{equation*} 
					 \arraycolsep=1.5pt\def \arraystretch{1.75}
			\begin{array}{lcl}
					&&384m^3\zeta^2T +  384m  \sum_{a=0}^{T-1}\mathbb{E}\left\|\nabla f\left(\boldsymbol{\varpi}^{\alpha}\right)\right\|^2  
						\\
						&+&  {128m^2c^2} \sum_{a=0}^{T-1}  \mathbb{E}\left\|\P^{\alpha}\right\|^2 \\
						&{\geq}&\Big(c^2(1-{\tau})^2-384 m^2\ell^2 \Big)\sum_{a=0}^{T-1}\sum_{q=1}^m \mathbb{E}\left\|\boldsymbol{\varpi}^{\alpha}-\overline{\w}_q^{\alpha }\right\|^2\\   
						  &\overset{\eqref{lower-bd-c-tau}}{\geq}& \frac{c^2(1-{\tau})^2}{2}\sum_{a=0}^{T-1}\sum_{q=1}^m \mathbb{E}\left\|\boldsymbol{\varpi}^{\alpha}-\overline{\w}_q^{\alpha }\right\|^2. 
					\end{array}
				\end{equation*}
This results in
				\begin{equation}\label{lemma2p1}
					 \arraycolsep=1.5pt\def \arraystretch{1.75}
			\begin{array}{lcl}
						 &&\sum_{\alpha=0}^{T-1}\sum_{p=1}^m \mathbb{E}\left\|\boldsymbol{\varpi}^{\alpha}-\overline{\w}_p^{\alpha }\right\|^2\\ 
						&\leq &  \frac{800m}{c^2(1-\tau)^2} ( {m^2\zeta^2T}  +    \sum_{a=0}^{T-1}\mathbb{E}\left\|\nabla f\left(\boldsymbol{\varpi}^{\alpha}\right)\right\|^2 \\ 
						&+&  \frac{mc^2}{3} \sum_{a=0}^{T-1}  \mathbb{E}\left\|\P^{\alpha}\right\|^2 ).
					\end{array}
				\end{equation}				 
			\end{proofpart}
\end{proofpart}		
\noindent	Finally, we note $ \mathbb{E}\A^\alpha$ is a doubly stochastic matrix. Consequently,  $\frac{\W^{\alpha}\mathbb{E}\A^\alpha\om}{m} = \frac{\W^{\alpha}\om}{m}=\boldsymbol{\varpi}^\alpha$, which results in		 
			\begin{equation}\label{lemma2p2}
				 \arraycolsep=1.5pt\def \arraystretch{1.25}
			\begin{array}{lcl}
					 \frac{m \ell}{2}\mathbb{E}\left\|\frac{\W^{\alpha}(\A^{\alpha}-\mathbb{E}\A^\alpha)\om}{m}\right\|^2
					&{=}&\frac{m \ell}{2}\mathbb{E}\left\|\frac{\W^\alpha\A^\alpha\om}{m}-\frac{\W^\alpha\om}{m}\right\|\\
					&{=}&\frac{m \ell}{2}\mathbb{E}\left\|\frac{1}{m}\sum_{p=1}^{m}\left(\boldsymbol{\varpi}^{\alpha}-\overline{\mathbf{w}}_p^{\alpha }\right)\right\|^2 \\
					 &\overset{(\ref{basiceq})}{\leq}& \frac{\ell}{2}\sum_{p=1}^{m}\mathbb{E}\left\|\boldsymbol{\varpi}^{\alpha}-\overline{\mathbf{w}}_p^{\alpha }\right\|^2.
				\end{array}
			\end{equation}	
Combining (\ref{lemma2p1}), we obtain		 
		\begin{equation*}
			 \arraycolsep=1.5pt\def \arraystretch{1.75}
			\begin{array}{lcl}
				&&\frac{1}{2 c m} \sum_{a=0}^{T-1}\mathbb{E}\left\|\nabla f\left(\boldsymbol{\varpi}^{\alpha}\right)-\Ps^{\alpha } \om\right\|^2\\ 
				&+&\frac{m\ell}{2}\sum_{a=0}^{T-1}\mathbb{E}\left\|\frac{\W^{\alpha}(\A^{\alpha}-\mathbb{E}{\A}^{\alpha})\om}{m}\right\|^2 \\
				&\overset{(\ref{parta-fact-1}),(\ref{lemma2p2})}{\leq} & \left(\frac{\ell^2}{2c}+\frac{\ell}{2}\right)
				\sum_{a=0}^{T-1}\sum_{p=1}^m \mathbb{E}\left\|\boldsymbol{\varpi}^{\alpha}-\overline{\w}_p^{\alpha }\right\|^2\\  &\overset{(\ref{lower-bd-c-tau})}{\leq}& \ell \sum_{a=0}^{T-1}\sum_{p=1}^m \mathbb{E}\left\|\boldsymbol{\varpi}^{\alpha}-\overline{\w}_p^{\alpha }\right\|^2   \\
				&\overset{(\ref{lemma2p1})}{\leq}& 
				\frac{800m\ell}{c^2(1-\tau)^2} ( {m^2\zeta^2T}  +    \sum_{a=0}^{T-1}\mathbb{E}\left\|\nabla f\left(\boldsymbol{\varpi}^{\alpha}\right)\right\|^2\\  
						&+&  \frac{mc^2}{3} \sum_{a=0}^{T-1}  \mathbb{E}\left\|\P^{\alpha}\right\|^2).
			\end{array}
		\end{equation*}
	 
		The whole proof is finished.
	\end{proof}

	\begin{lemma}\label{lemmalast} Under Assumption 1 and setting  ${c\geq c_0}$, it holds
		\begin{equation}
			 \arraycolsep=1.5pt\def \arraystretch{1.5}
			\begin{array}{lcl}
			 &&\frac{1}{4c m} \mathbb{E}\left\|\nabla f\left(\boldsymbol{\varpi}^{\alpha}\right)\right\|^2 \\
			  &\leq&	\mathbb{E}f\left(\boldsymbol{\varpi}^{\alpha}\right)  - \mathbb{E} f\left(\boldsymbol{\varpi}^{\alpha+1}\right)+ \frac{3c}{2}\mathbb{E}\left\|\P^{\alpha }\right\|^2\\
				&+&\frac{1}{2 c m} \mathbb{E}\left\|\nabla f\left(\boldsymbol{\varpi}^{\alpha}\right)-\Ps^{\alpha } \om\right\|^2+\frac{m\ell}{2}\left\|\frac{\W^{\alpha}\triangle\A^{\alpha}\om}{m}\right\|^2.
			\end{array}
			\label{presummation}
		\end{equation}
	\end{lemma}
	\begin{proof} Firstly, we have 
	\begin{equation}\label{E-f-avg-w}
			 \arraycolsep=1.5pt\def \arraystretch{1.75}
			\begin{array}{lcl}
				&&\mathbb{E}\left\langle\nabla f\left(\boldsymbol{\varpi}^{\alpha}
				- \frac{\Ps^{\alpha }\om}{c m}\right),\frac{\P^{\alpha }\om}{m}\right\rangle \\
				&\overset{(\ref{basiceq})}{\leq}& \frac{5cm}{4}\mathbb{E}\left\|\frac{\P^{\alpha }\om}{m}\right\|^2 + \frac{1}{5cm}\mathbb{E}\left\|\nabla f\left(\boldsymbol{\varpi}^{\alpha}
				- \frac{\Ps^{\alpha }\om}{c m}\right)\right\|^2
				 \\
				&\overset{(\ref{basiceq})}{\leq}&\frac{5cm}{4}\mathbb{E}\left\|\frac{\P^{\alpha }\om}{m}\right\|^2 
				 +\frac{1}{4cm}\mathbb{E}\left\|\nabla f\left(\boldsymbol{\varpi}^{\alpha}\right)\right\|^2 \\
				& +&\frac{1}{cm}\mathbb{E}\left\|\nabla f\left(\boldsymbol{\varpi}^{\alpha} - \frac{\Ps^{\alpha }\om}{c m}\right)-\nabla f\left(\boldsymbol{\varpi}^{\alpha}\right)\right\|^2  \\ 
				&{\leq}&\frac{5cm}{4}\mathbb{E}\left\|\frac{\P^{\alpha }\om}{m}\right\|^2 +\frac{1}{4cm}\mathbb{E}\left\|\nabla f\left(\boldsymbol{\varpi}^{\alpha}\right)\right\|^2\\
				&+& \frac{1}{c}\sum_{i=1}^m\mathbb{E}\left\|\nabla f_i\left(\boldsymbol{\varpi}^{\alpha} - \frac{\Ps^{\alpha }\om}{c m}\right)-\nabla f_i\left(\boldsymbol{\varpi}^{\alpha}\right)\right\|^2 
			  \\
				&{\leq}&\frac{5cm}{4}\mathbb{E}\left\|\frac{\P^{\alpha }\om}{m}\right\|^2 + \frac{\ell^2}{c^3m} \mathbb{E}\left\|{\Ps^{\alpha }\om}\right\|^2 
				+\frac{1}{4cm}\mathbb{E}\left\|\nabla f\left(\boldsymbol{\varpi}^{\alpha}\right)\right\|^2.
			\end{array}
		\end{equation}  
	Then the Lipschitz continuity enables the following conditions, 
		\begin{equation}\label{descent}
			 \arraycolsep=1.5pt\def \arraystretch{1.75}
			\begin{array}{lcl}
				&&\mathbb{E}f (\boldsymbol{\varpi}^{\alpha}-\frac{\Ps^\alpha\om}{c m} ) \\ 
				&\overset{(\ref{gralip})}{\leq}&
				\mathbb{E}f\left(\boldsymbol{\varpi}^{\alpha}\right)  - \frac{1}{c m}\mathbb{E} \left\langle\nabla f\left(\boldsymbol{\varpi}^{\alpha}\right), \Ps^\alpha \om\right\rangle
				+ \frac{m\ell}{2} \mathbb{E}\left\|\frac{\Ps^\alpha\om}{c m}\right\|^2 \\
				&{=}&
				\mathbb{E}f\left(\boldsymbol{\varpi}^{\alpha}\right)+ \frac{m\ell}{2} \mathbb{E}\left\|\frac{\Ps^\alpha\om}{c m}\right\|^2+\frac{1}{2c m}\mathbb{E}\left\|\nabla f\left(\boldsymbol{\varpi}^{\alpha}\right) - \Ps^\alpha\om\right\|^2 \\ 
				&-&  \frac{1}{2c m}\mathbb{E}\left\|\nabla f(\boldsymbol{\varpi}^{\alpha})\right\|^2
				-\frac{1}{2c m}\mathbb{E}\left\|\Ps^\alpha \om\right\|^2		
				\\
			 &{=}& \mathbb{E}f\left(\boldsymbol{\varpi}^{\alpha}\right) + \frac{\ell-c }{2c^2 m}\mathbb{E}\left\|\Ps^\alpha\om\right\|^2 - \frac{1}{2c m}\mathbb{E}\left\|\nabla f\left(\boldsymbol{\varpi}^{\alpha}\right)\right\|^2 \\
			 &+& \frac{1}{2c m}\mathbb{E}\left\|\nabla f\left(\boldsymbol{\varpi}^{\alpha}\right) - \Ps^\alpha\om\right\|^2.
			\end{array}
		\end{equation}
		In addition, one can check that	 
			\begin{equation}\label{two-contants}
			 \arraycolsep=1.5pt\def \arraystretch{1.75}
			\begin{array}{lcl}
&&\frac{\ell^2}{c^3m}+\frac{\ell-c}{2c^2m}  
\overset{\eqref{lower-bd-c-tau}}{\leq}  \frac{1}{2c^2m} \left(\frac{2\ell^2}{2\ell}  +\ell-2\ell \right) 
=	0 ,\\ 
&&\frac{5mc+2m\ell}{4}\overset{\eqref{lower-bd-c-tau}}{\leq}  \frac{5mc+m c}{4} \leq \frac{3mc}{2}.
			\end{array} 
		\end{equation}
		Moreover, by observing that $\frac{1}{m} \W^{\alpha}\mathbb{E}\A^\alpha\om= \frac{1}{m}\W^{\alpha}\om=\boldsymbol{\varpi}^\alpha$ and $\triangle\A^{\alpha}:=\A^{\alpha}-\mathbb{E}{\A}^{\alpha}$, it follows
			\begin{equation*}
			 \arraycolsep=1.5pt\def \arraystretch{1.5}
			\begin{array}{lcl}
				&&\boldsymbol{\varpi}^{\alpha+1}\overset{(\ref{avgpoint-1})}{=}  \frac{ \W^{\alpha+1} \om}{m}  \overset{(\ref{commustep})}{=} ~ \frac{\W^{\alpha}\A^{\alpha} \om}{m} - \frac{\Ps^{\alpha }\om}{c m}+\frac{\mathbf{\P}^{\alpha }\om}{m}  		
				\\
			 &{=}& \frac{\mathbf{\W^{\alpha}\mathbb{E}{\A}}^{\alpha }\om}{m} - \frac{\Ps^{\alpha }\om}{c m}+\frac{\mathbf{\P}^{\alpha }\om}{m} + \frac{\W^{\alpha}\triangle\A^{\alpha}\om}{m}				
				\\
				 &{=}&\boldsymbol{\varpi}^\alpha- \frac{\Ps^{\alpha }\om}{c m}+\frac{\mathbf{\P}^{\alpha }\om}{m} + \frac{\W^{\alpha}\triangle\A^{\alpha}\om}{m},
			\end{array} 
		\end{equation*}
		which by $\boldsymbol{\nu}^\alpha:=\boldsymbol{\varpi}^\alpha- \frac{\Ps^{\alpha }\om}{c m}+\frac{\mathbf{\P}^{\alpha }\om}{m}$ and $\mathbb{E} \triangle\A^{\alpha}=0$ yields	 
		\begin{equation*}
			 \arraycolsep=1.5pt\def \arraystretch{1.75}
			\begin{array}{lcl}
				 &&\mathbb{E}f\left(\boldsymbol{\varpi}^{\alpha+1}\right) \\
				&\overset{(\ref{gralip})}{\leq}& \mathbb{E}f\left(\boldsymbol{\nu}^\alpha \right) 
				+\mathbb{E} \langle\nabla f\left(\boldsymbol{\nu}^\alpha\right), \frac{\W^{\alpha}\triangle\A^{\alpha}\om}{m} \rangle
				+ \frac{m\ell}{2}\left\|\frac{\W^{\alpha}\triangle\A^{\alpha}\om}{m}\right\|^2  \\
				& {=}& \mathbb{E}f\left(\boldsymbol{\nu}^\alpha \right) 
				 + \frac{m\ell}{2}\left\|\frac{\W^{\alpha}\triangle\A^{\alpha}\om}{m}\right\|^2  \\
				&\overset{(\ref{gralip})}{\leq}&\mathbb{E}f\left(\boldsymbol{\varpi}^{\alpha} - \frac{\Ps^{\alpha }\om}{c m}\right) 
				+ \mathbb{E}\left\langle\nabla f\left(\boldsymbol{\varpi}^{\alpha} - \frac{\Ps^{\alpha }\om}{c m}\right),\frac{\P^{\alpha }\om}{m}\right\rangle  \\ 
				&+&\frac{m\ell}{2}\left\|\frac{\W^{\alpha}\triangle\A^{\alpha}\om}{m}\right\|^2 + \frac{m\ell}{2}\mathbb{E}\left\|\frac{\P^{\alpha }\om}{m}\right\|^2 \\
				&\overset{(\ref{descent})}{\leq}&\mathbb{E} f\left(\boldsymbol{\varpi}^{\alpha}\right)+\frac{\ell-c}{2 c^2 m} \mathbb{E}\left\|\Ps^{\alpha } \om\right\|^2-\frac{1}{2 c m} \mathbb{E}\left\|\nabla f\left(\boldsymbol{\varpi}^{\alpha}\right)\right\|^2 \\
				&+&\frac{1}{2 c m} \mathbb{E}\left\|\nabla f\left(\boldsymbol{\varpi}^{\alpha}\right)-\Ps^{\alpha } \om\right\|^2 + \frac{m\ell}{2}\mathbb{E}\left\|\frac{\P^{\alpha }\om}{m}\right\|^2
				 \\
				 & +&
				 \mathbb{E}\left\langle\nabla f\left(\boldsymbol{\varpi}^{\alpha} - \frac{\Ps^{\alpha }\om}{c m}\right),\frac{\P^{\alpha }\om}{m}\right\rangle
				 +\frac{m\ell}{2}\left\|\frac{\W^{\alpha}\triangle\A^{\alpha}\om}{m}\right\|^2  \\
				&\overset{(\ref{E-f-avg-w})}{\leq}&\mathbb{E} f\left(\boldsymbol{\varpi}^{\alpha}\right)
				+\left(\frac{\ell^2}{c^3m}+\frac{\ell-c}{2c^2m}\right) \mathbb{E}\left\|\Ps^{\alpha } \om\right\|^2\\
				&-&\frac{1}{4c m} \mathbb{E}\left\|\nabla f\left(\boldsymbol{\varpi}^{\alpha}\right)\right\|^2 + \frac{1}{2 c m} \mathbb{E}\left\|\nabla f\left(\boldsymbol{\varpi}^{\alpha}\right)-\Ps^{\alpha } \om\right\|^2\\
				& +&
				\frac{5mc+2m\ell}{4}\mathbb{E}\left\|\frac{\P^{\alpha }\om}{m}\right\|^2 +\frac{m\ell}{2}\left\|\frac{\W^{\alpha}\triangle\A^{\alpha}\om}{m}\right\|^2\\
				&\overset{(\ref{two-contants})}{\leq}&\mathbb{E} f\left(\boldsymbol{\varpi}^{\alpha}\right) -\frac{1}{4c m} \mathbb{E}\left\|\nabla f\left(\boldsymbol{\varpi}^{\alpha}\right)\right\|^2 +\frac{3c}{2}\mathbb{E}\left\|\P^{\alpha }\right\|^2  \\
				&+&\frac{1}{2 c m} \mathbb{E}\left\|\nabla f\left(\boldsymbol{\varpi}^{\alpha}\right)-\Ps^{\alpha } \om\right\|^2
				+\frac{m\ell}{2}\left\|\frac{\W^{\alpha}\triangle\A^{\alpha}\om}{m}\right\|^2,
			\end{array} 
		\end{equation*} 
finishing the proof.
	\end{proof}
	\subsection{Proof of Theorem 1}\label{appendixC}
	\begin{proof}
		Summing (\ref{presummation}) from $\alpha=0$ to $\alpha = T-1$ gives
		\begin{equation*}
			 \arraycolsep=1.5pt\def \arraystretch{1.75}
			\begin{array}{lcl}
				&& \frac{1}{4c m} \sum_{a=0}^{T-1} \mathbb{E}\left\|\nabla f\left(\boldsymbol{\varpi}^{\alpha}\right)\right\|^2 \\ 
				&\overset{(\ref{presummation})}{\leq}& \sum_{a=0}^{T-1}  \Big(\mathbb{E}f\left(\boldsymbol{\varpi}^{\alpha}\right) -\mathbb{E}f\left(\boldsymbol{\varpi}^{\alpha+1}\right)+
				\frac{3c}{2}\mathbb{E}\left\|\P^{\alpha }\right\|^2\\
				&+& \frac{1}{2 c m} \mathbb{E}\left\|\nabla f\left(\boldsymbol{\varpi}^{\alpha}\right)-\Ps^{\alpha } \om\right\|^2 
				+\frac{m\ell}{2}\left\|\frac{\W^{\alpha}\triangle\A^{\alpha}\om}{m}\right\|^2 \Big)\\ 	
				&{=}& \mathbb{E}f\left(\boldsymbol{\varpi}^{0}\right) -\mathbb{E}f\left(\boldsymbol{\varpi}^{T}\right)+\sum_{a=0}^{T-1} \Big(  
				\frac{3c}{2}\mathbb{E}\left\|\P^{\alpha }\right\|^2 \\
				&+& \frac{1}{2 c m} \mathbb{E}\left\|\nabla f\left(\boldsymbol{\varpi}^{\alpha}\right)-\Ps^{\alpha } \om\right\|^2  +\frac{m\ell}{2}\left\|\frac{\W^{\alpha}\triangle\A^{\alpha}\om}{m}\right\|^2\Big)\\ 	
				&\overset{(\ref{lemmagradient})}{\leq} & f\left(0\right) - f ^* +\frac{3c}{2}\sum_{a=0}^{T-1}  
				\mathbb{E}\left\|\P^{\alpha }\right\|^2+\frac{800m\ell}{c^2(1-\tau)^2}\Big( {m^2\zeta^2T}  \\
				&+&    \sum_{a=0}^{T-1}\mathbb{E}\left\|\nabla f\left(\boldsymbol{\varpi}^{\alpha}\right)\right\|^2  
				+  \frac{mc^2}{3} \sum_{a=0}^{T-1}  \mathbb{E}\left\|\P^{\alpha}\right\|^2\Big) \\ 				
				&\overset{(\ref{lower-bd-c-tau})}{\leq}& f\left(0\right) - f ^* +    2c\sum_{a=0}^{T-1} \mathbb{E}\left\|\P^{\alpha }\right\|^2\\
				&+&  \frac{c_0m^2\zeta^2T}{8mc^2} +   \frac{1}{8cm}\sum_{a=0}^{T-1}\mathbb{E}\left\|\nabla f\left(\boldsymbol{\varpi}^{\alpha}\right)\right\|^2  
				,
			\end{array}
		\end{equation*}
which immediately delivers that
				\begin{equation*} 
			 \arraycolsep=1.5pt\def \arraystretch{1.75}
			\begin{array}{lcl}
				&& \frac{1}{8c m} \sum_{a=0}^{T-1} \mathbb{E}\left\|\nabla f\left(\boldsymbol{\varpi}^{\alpha}\right)\right\|^2 \\ 
				&{\leq}& f\left(0\right) - f ^* +  \frac{c_0m^2\zeta^2T}{8mc^2}   +  2c\sum_{a=0}^{T-1} \mathbb{E}\left\|\P^{\alpha }\right\|^2\\
				 &\overset{\eqref{bd-E-P}}{\leq }& f\left(0\right) - f ^* +  \frac{c_0m^2\zeta^2T}{8mc^2}   +  \frac{2k_0  e_{\infty} }{c}.
			\end{array}
		\end{equation*} 
This condition suffices to
				\begin{equation*} 
			 \arraycolsep=1.5pt\def \arraystretch{1.5}
			\begin{array}{lcl}
				&& \frac{1}{  T} \sum_{a=0}^{T-1} \mathbb{E}\left\|\nabla f\left(\boldsymbol{\varpi}^{\alpha}\right)\right\|^2\\  
				 &{\leq}&  \frac{8cm(f\left(0\right) - f ^*)}{T}+   \frac{c_0m^2\zeta^2+16mck_0  e_{\infty}}{c}, 
			\end{array}
		\end{equation*} 	
	as desired. 
	\end{proof}

    \subsection{Proof of Theorem 2}\label{appendixD}
    \begin{proof}
    		At each iteration $k$, node $i$ accesses its raw dataset only during the gradient computation $\nabla f_i(\overline{\mathbf{w}}_i^k)$, i.e., at steps $ak_0$. By Assumption 5, the sensitivity of this computation is $u_i$. According to the Gaussian mechanism \cite[Theorem 3.22]{dwork2014algorithmic}, adding Gaussian noise with variance $\varrho_i$ ensures that the gradient computation satisfies $(\varepsilon, \delta)$-differential privacy. Now let $\mathcal{A}^k({D})$ be the operations of each iteration $k$ of Algorithm 1. Then it can be decomposed as a series of parallel operations ${\{\mathcal{A}_{i}^k({D}_i):i\in{V}\}}$ on all nodes, namely, ${\mathcal{A}^k({D}) = (\mathcal{A}_{1}^k({D}_1),\mathcal{A}_{2}^k({D}_2),\cdots,\mathcal{A}_{m}^k({D}_m))}$. By the parallel composition theorem \cite{mcsherry2009privacy}, the overall privacy budget is not affected by nodes' number $m$ because each data point is accessed only once across the local operations.  As a result,  each iteration $\mathcal{A}^k({D})$  is also $(\varepsilon, \delta)$-differential private.
    	\end{proof}
        \subsection{Proof of Theorem 3}\label{appendixE}
	\begin{proof}
		During $ak_0$ iterations, the noises are added when only on steps ${k=0, k_0, \ldots, (a-1)k_0}$. Therefore, the algorithm with $ak_0$ steps can be considered as an adaptive composition of $(\mathcal{A}_0({D}),\mathcal{A}_1({D}),\cdots,\mathcal{A}_{a-1}({D}))$, which by \cite[Theorem 3.2]{dwork2014algorithmic}   enables the conclusion.
	\end{proof}	
\end{document}